\documentclass{article}

\usepackage{arxiv}
\usepackage[utf8]{inputenc}  
\usepackage[T1]{fontenc}   
\usepackage{hyperref}        
\usepackage{url}             
\usepackage{booktabs}      
\usepackage{amsfonts}    
\usepackage{nicefrac}       
\usepackage{microtype}    
\usepackage{lipsum}
\usepackage{graphicx}
\usepackage{amsmath}
\usepackage{amsthm}
\usepackage{bbm}
\usepackage{algorithm}
\usepackage{algorithmic}
\usepackage{times}
\usepackage{amsmath}
\usepackage{hyperref}
\usepackage{url}
\usepackage{algorithmic}
\usepackage{algorithm}
\usepackage{amssymb}
\usepackage{graphicx}
\usepackage{epsfig}
\usepackage{subfigure}
\usepackage{lipsum}
\usepackage{footnote}
\usepackage{enumitem}
\usepackage{array}
\usepackage{multirow}
\usepackage{changepage,changepage}

\newtheorem{theorem}{Theorem}

\newtheorem{proposition}[theorem]{Proposition}

\title{Shift of Pairwise Similarities for Data Clustering}
\author{ Morteza Haghir Chehreghani  \\
                Department of Computer Science and Engineering\\
                Chalmers University of Technology\\
                SE-412 96 Gothenburg, Sweden\\
                \texttt{morteza.chehreghani@chalmers.se}\\
}

\begin{document}
\maketitle

\begin{abstract}
Several clustering methods (e.g., \emph{Normalized Cut} and \emph{Ratio Cut}) divide the \emph{Min Cut} cost function by a cluster dependent factor (e.g., the size or the degree of the clusters), in order to yield a more balanced partitioning. We, instead, investigate adding such regularizations to the original cost function. We first consider the case where the regularization term is the sum of the squared size of the clusters,  and then generalize it to adaptive regularization of the pairwise similarities. This leads to shifting (adaptively) the pairwise similarities which might make some of them negative. We then study the connection of this method to \emph{Correlation Clustering} and then propose an  efficient \emph{local search} optimization algorithm with fast theoretical convergence rate to solve the new clustering problem. In the following, we investigate the shift of pairwise similarities on some  common clustering methods, and finally, we demonstrate the superior performance of the method by extensive experiments on different datasets.


\end{abstract}

\section{Introduction}

Given a set of objects, clustering is concerned with grouping them in such a way that objects of the same group are more similar to each other (according to a predefined similarity measure), compared to those in different groups.
This task plays a fundamental role in several data analytics applications. Examples are  image segmentation (to detect the items in images),  document clustering (for the purpose of document organization, topic identification or efficient information retrieval), data compression, and analysis of (e.g., transportation) networks and graphs.  Clustering itself is not a specific method, rather is a general machine learning task to be addressed. The task can be solved via several methods that differ significantly in the way they define the notion of clusters and the way they extract them.
The concept of clustering is originated from anthropology and then was used in psychology~\cite{tryon1939cluster,Bailey1994}, in particular for trait theory classification in personality psychology~\cite{Cattell1943}.

A wide range of clustering methods introduce a cost function whose minimal solution provides a clustering solution. \emph{$K$-means} is a common cost function which is defined by the within-cluster sum of squared distances from the means~\cite{Macqueen67somemethods}. The data can be demonstrated by a graph, whose nodes represent the objects and the edge weights are the pairwise similarities between the objects. Then, a wide range of different graph partitioning methods can be applied to produce the clusters. Arguably, the most basic graph-based method is the \emph{Min Cut} (\emph{Minimum $K$-Cut}) cost function \cite{MinCutLeighton,MinCut-Wu}, in which the goal is to  partition the graph into exactly $K$ connected components (clusters) such that the sum of the inter-clusters edge weights is minimal. As we will see, the \emph{Min Cut} cost function often yields separating singleton clusters, in particular when the clusters have diverse densities. To overcome such problem, several clustering methods normalize the \emph{Min Cut} clusters to render more balanced clusters. For example, they propose to normalize the \emph{Min Cut} clusters by the size of the clusters (\emph{Ratio Assoc}~\cite{HofmannB97} and \emph{Ratio Cut}~\cite{ChanSZ94}) or the degree of  the clusters (\emph{Normalized Cut}~\cite{JS:JM:PAMI:2000}).

We note that balanced clustering has been studied for featured-based (vectorial) data as well, in particular with the \emph{$K$-means} method.
The method in \cite{balancedKmean2014} develops balanced clustering via formulating it as a assignment problem by the Hungarian algorithm, which suffers from a high runtime (cubic w.r.t. the number of objects). Another work models this problem as a least square linear regression with a balance constraints and uses the method of augmented Lagrange multipliers to solve it \cite{LiuHNL17}. The work in \cite{8621917} considers \emph{$K$-means} as the main clustering method and the respective cluster variances as the penalty term.  Then, \cite{ijcai2019-414} yields balanced clustering with convex regularization which makes the optimization more efficient. In the following, \cite{DING202028} studies balanced \emph{$K$-center}, \emph{$K$-median}, and \emph{$K$-means} in high dimensions with theoretical approximate algorithms. Finally, \cite{8490741} proposes a  balanced clustering framework that  utilizes both local and global information.
However, in this paper, we consider the `graph-based' balanced clustering variant, where we assume the clustering is applied to a given graph, instead of data features.

While most of graph clustering cost functions assume a nonnegative matrix of pairwise similarities as input, \emph{Correlation Clustering} assumes that the similarities can be negative as well. This  cost function was first introduced on graphs with only $+1$ and $-1$ edge weights~\cite{BansalBC04}, and  then it was generalized to graphs with arbitrary positive and negative edge weights~\cite{DemaineEFI06}.

Such graph clustering cost functions are often NP-hard~\cite{JS:JM:PAMI:2000,BansalBC04,DemaineEFI06}. However, the respective optimal solution can be approximated in some way. A category of methods work based on eigenvector analysis of the Laplacian matrix. \emph{Spectral Clustering}~\cite{JS:JM:PAMI:2000,Ng01onspectral} was the first method which exploits the information from eigenvectors. It forms a low-dimensional embedding by the bottom eigenvectors of the Laplacian of the similarity matrix and then applies $K$-means to produce the final clusters. A more recent method, called \emph{Power Iteration Clustering} (PIC)~\cite{LinC10PIC}, instead of embedding the data into a $K$-dimensional space, approximates an eigenvalue-weighted linear combination of all the eigenvectors of the normalized similarity matrix via early stopping of the power iteration method. \emph{P-Spectral Clustering} (PSC)~\cite{Buhler:2009PSpec} is another spectral approach that proposes a non-linear generalization of the Laplacian and then performs an iterative splitting method based on its second eigenvector.

An alternative graph-based clustering approach has been developed in the context of discrete time dynamical systems and evolutionary game theory which is based on performing replicator dynamics~\cite{Pavan:2007,NgMA12,LiuLY13}. \emph{Dominant Set Clustering} (DSC)~\cite{Pavan:2007} is an iterative method which at each iteration, peels off a cluster by performing a replicator dynamics until its convergence. The method in~\cite{LiuLY13} proposes an iterative clustering algorithm in two shrink and expansion steps, which helps to extract many small and dense clusters in large datasets. 
The method in~\cite{BuloPB11}, called \emph{InImDyn}, instead of replicator dynamics, suggests to use a population dynamics motivated from the analogy with infection and immunization processes within a population of players.

In this paper,  we investigate adding the regularization terms to the \emph{Min Cut} cost function, in order to avoid creation of small singleton sets of clusters. We first consider the case where the regularization is the sum of the squared size of the clusters, weighted by the parameter $\alpha$. This regularization leads to a simple shift transformation of the input, i.e.,  subtracting the pairwise similarities by $\alpha$, which provides a straightforward quadratic  cost function. We further extend the regularization to the pairwise similarities and employ an adaptive shift of the pairwise similarities which does not require fixing a regularization parameter in advance.
The size constrained \emph{Min Cut} then constitutes a special case of the latter form.
 Such a shift might render some pairwise similarities to be negative. We then study the connection to \emph{Correlation Clustering}, another cost function which performs on both positive and negative similarities, and  conclude the equivalence of these two methods given the shifted (regularized) pairwise similarities in a direct and straightforward way (beyond the argument based on algorithmic reduction proposed in \cite{DemaineEFI06}).
However, our method, called \emph{Shifted Min Cut}, provides a principled way to deduce  such negative edge weights (adaptively). Thereafter, we develop an efficient optimization method based on \emph{local search} to solve the new optimization problem. We further discuss the fast theoretical convergence rate of this local search algorithm.
 In the following, we study the impact of shifting the pairwise similarities on some  common flat and hierarchical clustering methods  where they often exhibit an invariant behaviour with respect to the shift of pairwise similarities, unlike the basic \emph{Min Cut} cost function. Finally, we perform extensive experiments on several real-world datasets to study the performance of \emph{Shifted Min Cut} compared to the alternatives.

This work is an extension of our previous work  \cite{ChehreghaniICDM17} wherein we additionally, i) provide an argument on the theoretical convergence rate of the local search algorithm based on the connection to an optimized variant of Frank-Wolfe algorithm, ii) discuss the shift of pairwise similarities on several other clustering methods, and iii) elaborate further the existing experimental results and perform extra  studies on real-world datasets.
We have later found out that the work in \cite{ChenZJ05} suggests a similar idea for regularization of \emph{Min Cut} in order to yield balanced clusters. However, there are several fundamental differences between \cite{ChenZJ05} and our work: i) they study size constrained \emph{Min Cut} for bi-partitioning (i.e., for only two clusters), whereas we model it for arbitrary $K$ clusters. Then, to generate more clusters than two, they propose an iterative (sequential) bi-partitioning which might cause the re-scaling problem. ii) Their method requires fixing  critical hyperparameters often in a heuristic way, whereas our method does not include  such hyperparameters. iii) Beyond size constrained \emph{Min Cut}, we extend the method to refined regularization of the pairwise similarities that yields an adaptive regularization (shift) of the cost function. This adaptive regularization, not only provides adaptivity with respect to the type of the relations, but also obviates the need for fixing critical hyperparameters. iv) We consider that the regularization renders some of the pairwise similarities to be negative, and thereby, we study the connection between such a regularized (\emph{Shifted}) \emph{Min Cut} method and \emph{Correlation Clustering}. However, \cite{ChenZJ05} does not study such a connection. v) To optimize the respective cost function, we employ  integration of the regularizations into shifting the pairwise similarities and develop an efficient local search algorithm that enjoys a linear convergence rate.
\cite{ChenZJ05}, instead, develops approximate spectral solutions. vi) We demonstrate the performance of the method on several real-world datasets with respect to different evaluation criteria, whereas \cite{ChenZJ05} only studies the  mutual information evaluation criterion on two datasets. In particular, we investigate both the cost function and its optimization separately.

The rest of the paper is organized as following. In section 2, we introduce the notations and the definitions. Then, in section 3, we describe the regularization and the connection to shifting the pairwise similarities. In this section, we  extend the method to adaptive regularization (shift) of the pairwise similarities. In section 4, we study the connection between \emph{Shifted Min Cut} and \emph{Correlation Clustering}, and, in section 5,  we develop an efficient local search optimization method for the cost function. In section 6, we study the consequence of shifting pairwise relations in some other (flat and hierarchical) clustering methods. In section 7, we experimentally investigate the different aspects of the method on several real-world datasets, and finally, in section 8, we conclude the paper.

\section{Notations and Definitions}

The data is given by a set of $n$ objects $\mathbf O=\{1,...,n\}$ and the corresponding matrix of pairwise similarities $\mathbf X = \{\mathbf X_{ij}\}, \forall i,j \in \mathbf O$. Thus, the data can be represented by (an undirected) graph $\mathcal G(\mathbf O,\mathbf X)$, where the objects $\mathbf O$ constitute the nodes of the graph and $\mathbf X_{ij}$ represents the weight of the edge  between $i$ and $j$.   Then,  the goal is to partition the objects (the graph) into $K$ coherent groups which are distinguishable from each other. The clustering solution is encoded in $\mathbf c\in \{1,...,K\}^n$, i.e., $\mathbf c_i$ indicates the cluster label of the $i^{th}$ object. 
The vector $\mathbf c$ can be also represented via the co-clustering matrix $\mathbf H \in \{0,1\}^{n \times n}$.
\begin{equation}
\mathbf H_{ij} = \begin{cases} 1 & \mbox{iff } \mathbf c_i=\mathbf c_j \\ 0, & \mbox{otherwise.}\end{cases}
\end{equation}
$\mathcal C$ denotes the space of all different clustering solutions.

Moreover, we assume $\mathbf O_k\subset \mathbf O$ includes the members of the $k^{\text{th}}$ cluster, i.e.,
\begin{eqnarray}
    \mathbf O_k := \{i\in\mathbf O\;:\;\mathbf c_i=k\}\; .
\end{eqnarray}

$|\mathbf O_k|$ refers to the size of the $k^{\text{th}}$ cluster.

\section{Shift of Pairwise Similarities for Clustering}

Different graph-based clustering methods often consider the \emph{Min Cut} cost function as a base method which is defined by
\begin{eqnarray}
    R^{MC}(c,\mathbf X) &=& \sum_{k=1}^{K} \sum_{\substack{k'=1,\\k'\ne k}}^{K} \sum_{i\in\mathbf O_k} \sum_{j \in\mathbf O_{k'}} \mathbf X_{ij} \, .
    \label{eq:minCutKway}
\end{eqnarray}
This cost function has a tendency to split small sets of objects, since the cost increases with the number of inter-cluster edge weights, i.e., the edges connecting the different clusters. Figure~\ref{fig:MinCutMalik}  illustrates such a situation for two clusters~\cite{JS:JM:PAMI:2000}. We assume that the edge weights are inversely proportional to the distances between the objects. It is observed that \emph{Min Cut} favors splitting objects $i$ or $j$, instead of performing a more balanced split. In fact, any cut that splits one of the objects on the right half will yield a smaller cost than the cut that partitions the objects into the left and right halves. This issue is particularly problematic when the intra-cluster edge weights are heterogeneous among different clusters. 
Thus, several methods propose to normalize the \emph{Min Cut} clusters by a cluster depending factor, e.g., the size of clusters (\emph{Ratio Assoc}~\cite{HofmannB97} and \emph{Ratio Cut}~\cite{ChanSZ94}) or the degree of clusters (\emph{Normalized Cut}~\cite{JS:JM:PAMI:2000}).
\begin{figure}[t!]
    \centering
    \includegraphics[width=0.35\textwidth]{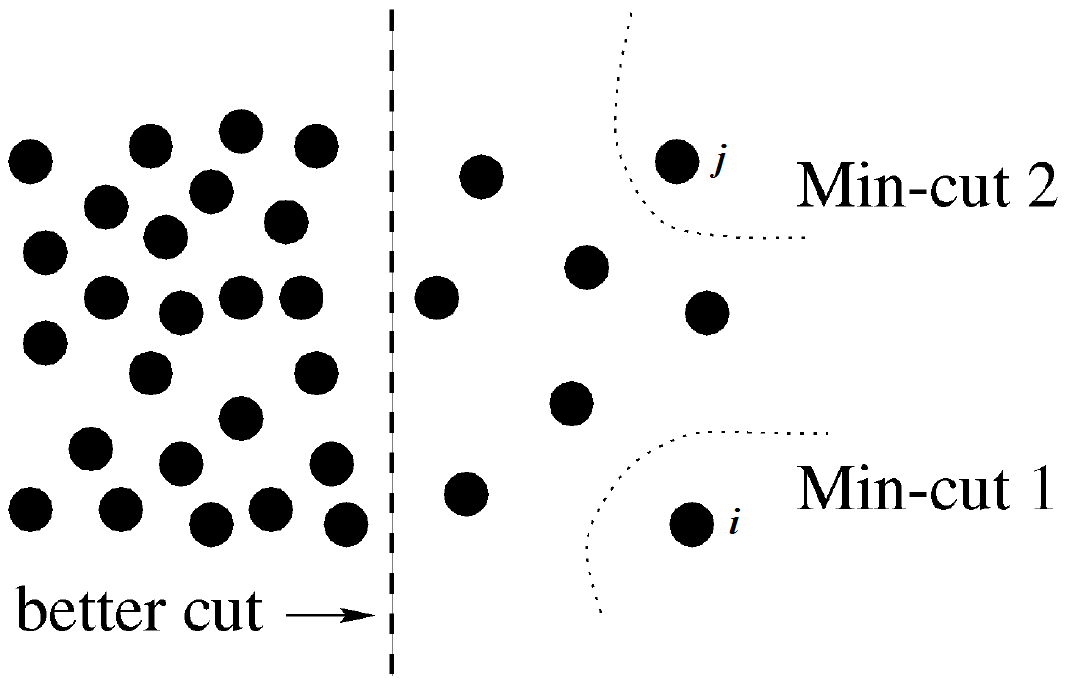}
\caption{The \emph{Min Cut} cost function has a bias to split small (singleton) sets of objects. Any cut that splits one of the objects on the right half will have smaller cost than the cut that splits the objects into the left and right halves. The figure has been adapted from~\cite{JS:JM:PAMI:2000}.
\label{fig:MinCutMalik} }
\end{figure}

We investigate an alternative approach to yield the occurrence of more balanced clusters. Instead of normalizing (dividing) the \emph{Min Cut} cost function by a cluster-dependent function, we consider adding such a regularization to the original cost function, i.e.,

\begin{equation}
    R^{new}(\mathbf c, \mathbf X,\alpha)  =  R^{MC}(\mathbf c,\mathbf X) + \alpha \,.\, r(\mathbf c, \mathbf X) ,
\end{equation}

where $r(\mathbf c, \mathbf X)$ indicates the regularization.
Note that this formulation involves the two free choices $\alpha$ and $ r(\mathbf c, \mathbf X)$, thereby, it yields a richer family of alternative methods. We first focus on the case where  $r(\mathbf c, \mathbf X)$ is the sum of the squared size of the  clusters\footnote{The \emph{Min Cut} cost function is quadratic with respect to the number of edges, therefore, to be consistent, we use the squared form of the cluster cardinalities.}, i.e.,

\begin{equation}
    R^{new}(\mathbf c, \mathbf X) = R^{MC}(\mathbf c,\mathbf X) + \alpha \sum_{k=1}^{K} |\mathbf O_k|^2 \; .
\label{eq:newMincut}
\end{equation}

Thereby,
\begin{enumerate}
  \item if $\alpha<0$, then the term $\alpha \sum_{k=1}^{K} |\mathbf O_k|^2$ is minimal when only the singleton clusters (objects) are separated. Thus, this choice does not help to avoid occurrence of singleton clusters, rather, it accelerates.
  \item If $\alpha>0$, then $\alpha \sum_{k=1}^{K} |\mathbf O_k|^2$ is minimal for balanced clusters, i.e., when $|\mathbf O_k|\approx n/K \ , \forall k\in \{1,...,K\}$. This leads to equalize the size of clusters. We note that $|\mathbf O_k|$'s are integer numbers, but $n/K$ is not necessarily an integer. Thus, we may arbitrarily set some of the $|\mathbf O_k|$'s to $\left \lceil{n/K}\right \rceil$  and some others to $\left \lfloor{n/K}\right \rfloor$ such that $\sum_{k=1}^{K} |\mathbf O_k|=n$. The order would not change the minimum.
\end{enumerate}

The cost function in Eq.~\ref{eq:newMincut} can be further written as

\begin{align}
    &R^{new}(\mathbf c, \mathbf X, \alpha) = R^{MC}(\mathbf c,\mathbf X) + \alpha \sum_{k=1}^{K} |\mathbf O_k|^2 \nonumber\\
    &\quad = \sum_{k=1}^{K} \sum_{k'\ne k}^{K} \sum_{i\in \mathbf O_k} \sum_{j\in \mathbf O_{k'}} \mathbf X_{ij} + \alpha \sum_{k=1}^{K} |\mathbf O_k|^2  \nonumber\\
    &\quad =\sum_{k=1}^{K} \sum_{k'\ne k}^{K} \sum_{i\in \mathbf O_k} \sum_{j\in \mathbf O_{k'}} \mathbf X_{ij} + \sum_{k=1}^K \sum_{i ,j\in \mathbf O_k}  \mathbf X_{ij}
     - \sum_{k=1}^K \sum_{i ,j\in \mathbf O_k}  \mathbf X_{ij} + \sum_{k=1}^K \sum_{i ,j\in \mathbf O_k}  \alpha \nonumber\\
	&\quad =\sum_{k=1}^{K} \sum_{k'=1}^{K} \sum_{i\in \mathbf O_k} \sum_{j\in \mathbf O_{k'}} \mathbf X_{ij} - \sum_{k=1}^K \sum_{i ,j\in \mathbf O_k}  (\mathbf X_{ij}-\alpha) \nonumber\\
    &\quad = \underbrace{\sum_{i ,j\in \mathbf O}  \mathbf X_{ij}}_{constant} - \sum_{k=1}^K \sum_{i ,j\in \mathbf O_k}  (\mathbf X_{ij}-\alpha) \nonumber\\
    &\quad = - \sum_{k=1}^K \sum_{i ,j\in \mathbf O_k}  (\mathbf X_{ij}-\alpha) + constant\, .
\label{eq:sMinCut}
\end{align}

Therefore, we define

\begin{center}
\fbox{$R^{SMC}(\mathbf c, \mathbf X, \alpha) = - \sum_{k=1}^K \sum_{i ,j\in \mathbf O_k} (\mathbf X_{ij}-\alpha)$.}
\end{center}

Thus, we employ a shifted variant of \emph{Min Cut} cost function (called \emph{Shifted Min Cut}), wherein all pairwise similarities are subtracted by a positive parameter $\alpha$, such that some of the pairwise similarities might become negative.
It makes sense that the regularization on the size of the clusters becomes connected to the pairwise similarities, as, at the end, pairwise relations are  responsible for creating the clusters. Thus, by tuning them properly, one should be able to obtain the desired balanced clusters. Thereby, the cluster level regularization is effectively applied to the representation space, where, as will be discussed, it yields modeling and computational advantages.

\begin{figure}[t!]
    \centering
    \includegraphics[width=0.5\textwidth]{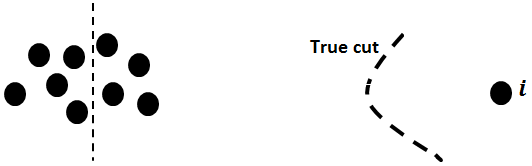}
\caption{The impact of the shift parameter $\alpha$ on the results of the \emph{Shifted Min Cut} cost function. A very large $\alpha$ might yield splitting large clusters, instead of separating true small clusters.
\label{fig:MinCutNegative} }
\end{figure}

This formulation provides a rich family of alternative clustering methods where different regularizations are induced by different values of $\alpha$. However, choosing a very large $\alpha$ can lead to equalizing the size of the clusters that are inherently very unbalanced in size. For example, consider the dataset shown in Figure~\ref{fig:MinCutNegative}. We assume that the edge weights  are inversely proportional to the pairwise distances. Then, we subtract all pairwise similarities by a very large number.
Therefore, the pairwise similarities become very large but negative numbers which renders  \emph{Shifted Min Cut}  to produce equal-size clusters, even though a correct cut should separate only the object $i$ from the rest.
Thus, in practice one needs to examine different  values of $\alpha$, and choose the one that yields the best results, or is preferred by the user. However, this procedure might be computationally expensive, and, moreover, the user might not be able to validate the correct solution among many different alternatives, due to lack of enough prior knowledge, supervision or side information.
For this reason, we employ a particular shift of pairwise similarities which takes the connectivity of the objects into account and does not need fixing any free parameter. 
\paragraph{Adaptive shift of pairwise similarities.} Different pairwise similarities might need different shifts, depending on the type and the density of the clusters that the respective objects belong to.  Therefore, we relax the constraints of the formulation in Eq. \ref{eq:sMinCut}  and consider a separate shift parameter for every pairwise similarity $\mathbf{X}_{ij}$.

\begin{equation}\label{eq:sMinCut_adaptive}
  R^{SMC}(\mathbf c, \mathbf X, \{\alpha_{ij}\}) = - \sum_{k=1}^K \sum_{i ,j\in \mathbf O_k} (\mathbf X_{ij}-\alpha_{ij}), \qquad \alpha_{ij} = \alpha_{ji}  \, .
\end{equation}

The formulation in Eq. \ref{eq:sMinCut_adaptive} already involves  the formulation in Eq. \ref{eq:sMinCut} as a special case  where all $\alpha_{ij}$'s are fixed by a constant. 
To determine $\alpha_{ij}$'s properly, a reasonable approach is to shift the pairwise similarity $\mathbf X_{ij}$ between $i$ and $j$ adaptively with respect to the similarities between $i$ and all the other objects and as well as the similarities between $j$ and the other objects.
For this purpose, we shift $\mathbf X_{ij}$ such that the sum of the pairwise similarities between $i$ and all the other objects becomes zero, and the same holds for $j$ too.
In this way, we have

\begin{equation}
\alpha_{ij} =  \frac{1}{n} \sum_{p=1}^{n} \mathbf X_{ip}
+ \frac{1}{n} \sum_{p=1}^{n} \mathbf X_{pj} - \frac{1}{n^2} \sum_{p=1}^n \sum_{q=1}^n \mathbf X_{pq} \, .
\label{eq:centralizedShift}
\end{equation}

Summing up the regularizations for all pairs of objects, we have (we assume $\mathbf X$ is symmetric):

\begin{eqnarray}
\sum_{k=1}^{K} \sum_{i,j \in \mathbf O_k} \alpha_{ij} &=& \sum_{k=1}^{K} \sum_{i,j \in \mathbf O_k} \left ( \frac{1}{n} \sum_{p=1}^{n} \mathbf X_{ip}
+ \frac{1}{n} \sum_{p=1}^{n} \mathbf X_{pj} - \frac{1}{n^2} \sum_{p=1}^n \sum_{q=1}^n \mathbf X_{pq} \right ) \nonumber \\
&=& \frac{2}{n} \sum_{k=1}^{K}|\mathbf O_k| deg(k) -  \frac{\beta}{n^2}  \sum_{k=1}^{K}|\mathbf O_k|^2 \, ,
\label{eq:adaptive_reg}
\end{eqnarray}

where $deg(k)$ is the degree of cluster $k$, i.e., $deg(k)=\sum_{i \in \mathbf O_k}\sum_{p=1}^{n} \mathbf X_{ip}$, and constant $\beta$ is the sum of the given pairwise similarities, i.e.,  $\beta=\sum_{p=1}^n \sum_{q=1}^n \mathbf X_{pq}$.  Therefore, the  adaptive regularization yields a  tradeoff between the size of the clusters and the degree of the clusters. The former is used in \emph{Ratio Assoc} and the latter in \emph{Normalized Cut}, both in the denominator. However, here a combination of these two is assumed, but as additive terms.

Therefore, the new shifted similarity $\mathbf S_{ij}$ is obtained by

\begin{equation}
\mathbf S_{ij} = \mathbf X_{ij} - \frac{1}{n} \sum_{p=1}^{n} \mathbf X_{ip}
- \frac{1}{n} \sum_{p=1}^{n} \mathbf X_{pj} + \frac{1}{n^2} \sum_{p=1}^n \sum_{q=1}^n \mathbf X_{pq} \, .
\label{eq:centralizedShift}
\end{equation}

It is easy to check that $\mathbf S$ is symmetric, provided that $\mathbf X$ is symmetric.
It can be shown that sum of the rows and the columns of $\mathbf S$ are equal to zero. For example, for a fixed row $i$ we have

\begin{align}
\mathbf \sum_{j=1}^{n} \mathbf S_{ij} =&\sum_{j=1}^{n} \mathbf X_{ij} - \frac{1}{n} \sum_{j=1}^{n} \sum_{p=1}^{n} \mathbf X_{ip}
- \frac{1}{n} \sum_{j=1}^{n} \sum_{p=1}^{n} \mathbf X_{pj} + \frac{1}{n^2} \sum_{j=1}^{n} \sum_{p=1}^n \sum_{q=1}^n \mathbf X_{pq}\nonumber \\
=&\sum_{j=1}^{n} \mathbf X_{ij} - \frac{n}{n} \sum_{p=1}^{n} \mathbf X_{ip}
- \frac{1}{n} \sum_{j=1}^{n} \sum_{p=1}^{n} \mathbf X_{pj} + \frac{n}{n^2} \sum_{p=1}^n \sum_{q=1}^n \mathbf X_{pq}\nonumber \\
& = 0 + 0 \, .
\label{eq:centralizedSum}
\end{align}

The adaptive shift in Eq.~\ref{eq:centralizedShift} can be written in matrix form as

\begin{equation}
\mathbf S = \mathbf T \mathbf X \mathbf T \, ,
\label{eq:ShiftMatrix}
\end{equation}
where the $n \times n$ matrix $\mathbf T$ is defined by
\begin{equation}
\mathbf T = \mathbf I_n  - \frac{1}{n}\mathbf U  \, .
\end{equation}

$\mathbf U$ is an $n \times n$ matrix whose all elements are $1$. \footnote{Due to shift invariance of \emph{Ratio Assoc}, a similar shift is used in \cite{roth03PAMI} to render the respective eigenvalues non-negative and thus obtain an embedding for the pairwise relations. However, here, the shift is used for a totally different purpose and it yields changing the size of clusters.}

Thus, according to Eqs. \ref{eq:sMinCut} and \ref{eq:sMinCut_adaptive}, the new cost function is written by
\begin{eqnarray}
R^{SMC}(\mathbf c, \mathbf S) &=& - \sum_{k=1}^K \sum_{i ,j\in \mathbf O_k}  \mathbf S_{ij} \label{eq:sMaxCor} \\
&\equiv& \sum_{k=1}^{K} \sum_{k'\ne k}^{K} \sum_{i\in \mathbf O_k} \sum_{j\in \mathbf O_{k'}} \mathbf S_{ij} \,  .
\label{eq:sMinCutTemp}
\end{eqnarray}
As an alternative to the adaptive shift, a proper shift can be obtained by investigating few pairwise relations by a user (i.e., a kind of weak supervision). In this setting, the user tells us how the actual pairwise relations should look like for a small subset of them, i.e., weather they are in the same cluster (positive shift) or different clusters (negative shift). Then, given this feedback, we can generalize them to all the pairwise relations. We may train a model, e.g. a neural network, which learns the shift depending on the specifications of the respective edge and objects. Such an approach can be even combined with our method for adaptive shift of pairwise similarities, where the later is used as an initial guess for the shifted pairwise relations and then they are fine tuned further using the user feedbacks if needed.
This formulation also provides a convenient way to encode constraints and prior knowledge such as ‘objects $x$ and $y$ must be together’, and ‘objects $p$ and $q$ must be in different clusters’.

\section{Relation to Correlation Clustering}

\emph{Correlation Clustering} is a clustering cost function that partitions a graph with positive and negative edge weights. The cost function sums the disagreements, i.e., the sum of negative intra-cluster edge weights plus the sum of positive inter-cluster edge weights. The respective cost function on general graphs is defined by~\cite{DemaineEFI06}

\begin{equation}
R^{CC}(c,\mathbf X) = \sum_{(i,j) \in E^{<+>}} \mathbf X_{ij}(1-\mathbf H_{ij}) - \sum_{(i,j) \in E^{<->}} \mathbf X_{ij} \mathbf H_{ij} \, ,
\label{eq:CC_Demaine}
\end{equation}
where $E^{<->}$ and $E^{<+>}$ respectively indicate the set of the edges with negative and with positive weights. The approximation scheme in  \cite{DemaineEFI06} reduces  \emph{Min Cut} to  \emph{Correlation Clustering} in order to obtain a logarithmic approximation factor for \emph{Correlation Clustering}. It also develops a reduction from \emph{Correlation Clustering} to \emph{Min Cut} to conclude the equivalence of these two cost functions.
Here, we elaborate that these two cost functions are identical and represent the same objective (given the shifted pairwise similarities) in a direct and straightforward way without using the more complicated reduction argument. In addition, \cite{DemaineEFI06} assumes that the number of clusters is hidden in the cost function (as defined in Eq. \ref{eq:CC_Demaine}). However, we study the equivalence for any  arbitrary number of clusters $K$. As shown in \cite{ChehreghaniBB12,FrankCB11}, optimizing \emph{Correlation Clustering} without a constraint on the number of clusters can lead to overfitting and unrobust solutions, whereas fixing the number of clusters may avoid these issues. Therefore, we consider the setting where the number of clusters $K$ is explicitly specified in the cost function and the user has the possibility to fix it in advance. Finally, the reduction-based argument in \cite{DemaineEFI06} yields the equivalence of the optimal solutions between \emph{Min Cut} and \emph{Correlation Clustering} and the respective approximation and hardness results. We, in addition, conclude the equivalence of any local optimal solution for the two cost functions, which is important when using \emph{local search} algorithms to optimize the cost functions.

For a fixed $K$, the \emph{Correlation Clustering} cost function can be written as  \cite{ChehreghaniBB12,FrankCB11}

\begin{align}
&R^{CC}(c,\mathbf X) = \underbrace{\frac{1}{2}\sum_{k=1}^K \sum_{i,j \in \mathbf O_k} (|\mathbf X_{ij}|-\mathbf X_{ij})}_{a}
+ \underbrace{\frac{1}{2}\sum_{k=1}^{K} \sum_{\substack{k'=1,\\k'\ne k}}^{K} \sum_{i\in\mathbf O_k} \sum_{j \in\mathbf O_{k'}} (|\mathbf X_{ij}|+ \mathbf X_{ij})}_{b} \, .
\label{eq:CC}
\end{align}

The first term (called $a$) sums the intra-cluster negative edge weights, whereas the second term (called $b$) sums the inter-cluster positive edge weights. We separately expand each term.

\begin{align}
a =& \frac{1}{2}\sum_{k=1}^K \sum_{i,j \in \mathbf O_k} |\mathbf X_{ij}| - \frac{1}{2}\sum_{k=1}^K \sum_{i,j \in \mathbf O_k} \mathbf X_{ij} \nonumber \\
=& \frac{1}{2}\sum_{k=1}^K \sum_{i,j \in \mathbf O_k} |\mathbf X_{ij}| - \underbrace{\frac{1}{2}\sum_{k=1}^{K} \sum_{k'=1}^{K} \sum_{i\in\mathbf O_k} \sum_{j \in\mathbf O_{k'}} \mathbf X_{ij}}_{constant}
+ \frac{1}{2}\sum_{k=1}^{K} \sum_{\substack{k'=1,\\k'\ne k}}^{K} \sum_{i\in\mathbf O_k} \sum_{j \in\mathbf O_{k'}} \mathbf X_{ij} \,.
\end{align}

Similarly, we expand term $b$.

\begin{eqnarray}
b &=& \frac{1}{2}\sum_{k=1}^{K} \sum_{\substack{k'=1,\\k'\ne k}}^{K} \sum_{i\in\mathbf O_k} \sum_{j \in\mathbf O_{k'}} |\mathbf X_{ij}|
+ \frac{1}{2}\sum_{k=1}^{K} \sum_{\substack{k'=1,\\k'\ne k}}^{K} \sum_{i\in\mathbf O_k} \sum_{j \in\mathbf O_{k'}} \mathbf X_{ij} \, .
\end{eqnarray}

Then, by summing $a$ and $b$ we obtain

\begin{align}
&R^{CC}(c,\mathbf X) = constant   \nonumber \\
& \quad + \underbrace{\frac{1}{2}\sum_{k=1}^K \sum_{i,j \in \mathbf O_k} |\mathbf X_{ij}| + \frac{1}{2}\sum_{k=1}^{K} \sum_{\substack{k'=1,\\k'\ne k}}^{K} \sum_{i\in\mathbf O_k} \sum_{j \in\mathbf O_{k'}} |\mathbf X_{ij}|}_{constant}
+ \underbrace{\sum_{k=1}^{K} \sum_{\substack{k'=1,\\k'\ne k}}^{K} \sum_{i\in\mathbf O_k} \sum_{j \in\mathbf O_{k'}} \mathbf X_{ij}}_{R^{MC}(\mathbf c, \mathbf X)} \,.
\end{align}

Thus, \emph{Correlation Clustering} and \emph{Min Cut} are equivalent cost functions, i.e.,
\begin{enumerate}
\item The cost functions share the same optimal solution, i.e., $\arg\min_{\mathbf c}R^{MC}(\mathbf c,\mathbf X) = \arg\min_{\mathbf c}R^{CC}(\mathbf c,\mathbf X)$.
\item The costs differences are the same, i.e., $\forall \mathbf c \in \mathcal C:  R^{MC}(\mathbf c,\mathbf X) - \min_{\mathbf c}R^{MC}(\mathbf c,\mathbf X) = R^{CC}(\mathbf c,\mathbf X) - \min_{\mathbf c}R^{CC}(\mathbf c,\mathbf X).$ This is in particular relevant when defining for example a Boltzmann distribution over the solution space $\mathcal C$.
\end{enumerate}
Thus, \emph{Correlation Clustering}, similar to \emph{Shifted Min Cut}, is an extension of  \emph{Min Cut}  which deals with both negative and positive edge weights. However, there are fundamental differences between these two methods:
\begin{enumerate}
  \item \emph{Correlation Clustering} assumes that the matrix of pairwise positive and negative similarities is given (which might be nontrivial), whereas \emph{Shifted Min Cut} proposes a principled way to yield clustering of positive and negative similarities via regularizing the base \emph{Min Cut} cost function. Thus, \emph{Shifted Min Cut} provides an explicit and straightforward interpretation of the clustering problem.
  \item The form of the \emph{Shifted Min Cut} cost function expressed in Eq. \ref{eq:sMaxCor} provides efficient function evaluations (e.g., for optimization) compared to the \emph{Correlation Clustering} cost function in Eq.~\ref{eq:CC} or the base \emph{Min Cut} cost function in Eq.~\ref{eq:minCutKway}. The cost functions in Eqs~\ref{eq:CC} and~\ref{eq:minCutKway} are quadratic with respect to $K$, the number of clusters, whereas the cost function in Eq.~\ref{eq:sMaxCor} is linear.
\end{enumerate}

\section{Optimization of the \emph{Shifted Min Cut} Cost Function}

Finding the optimal solution of the standard \emph{Min Cut}  with non-negative edge weights, i.e., when $\mathbf X_{ij}\ge 0, \forall i,j$, is  well-studied, for which there exist several polynomial time algorithms, e.g., $\mathcal O(n^4)$~\cite{goldschmidt94} and $\mathcal O(n^2 \log^3 n)$~\cite{Karger:1996}.
However, finding the optimal solution of the \emph{Shifted Min Cut} cost function, wherein some edge weights are negative, is NP-hard~\cite{BansalBC04,DemaineEFI06} and even is APX-hard~\cite{DemaineEFI06}.
Therefore, we develop a \emph{local search} method which  computes a local minimum of the cost function in Eq.~\ref{eq:sMaxCor}. The effectiveness of such a greedy strategy is well studied for different clustering cost functions, e.g., $K$-means~\cite{Macqueen67somemethods},  kernel $K$-means~\cite{Scholkopf:1998} and in particular several graph partitioning methods~\cite{Dhillon:2004,Dhillon:EtAl:05}.\footnote{Consistently, with \emph{Correlation Clustering} we observe a significantly better performance of the local search algorithm compared to approximation schemes such as those proposed in~\cite{BansalBC04,DemaineEFI06}.} In this approach, we start with a random clustering solution and then we iteratively assign each object to the cluster that yields a maximal reduction in the cost function. We repeat this procedure until no further change of assignments is achieved during a complete round of investigation of the objects, i.e., then a local optimal solution is attained.

At each iteration of the aforementioned procedure, one needs to evaluate the cost of assigning every object to each of the clusters. The cost function is quadratic, thus a single evaluation might take $\mathcal O(Kn^2)$ runtime. Thereby, if the local search converges after $t$ iterations, then, the total runtime will be $\mathcal O(tKn^3)$ for $n$ objects, which might be computationally expensive.

However, we do not need to recalculate the cost function for every individual evaluation.
Let $R^{SMC}(\mathbf c_{o \to l},\mathbf S)$ denote the cost of the clustering solution $\mathbf c$ wherein object $o$ is assigned to cluster $l$. At each step of the local search algorithm, we need to evaluate the cost $R^{SMC}(\mathbf c_{o \to l'},\mathbf S), l' \ne l$ given  $R^{SMC}(\mathbf c_{o \to l},\mathbf S)$.

The cost $R^{SMC}(\mathbf c_{o \to l},\mathbf S)$ is written by
\begin{equation}
R^{SMC}(\mathbf c_{o \to l},\mathbf S) = - \sum_{k=1}^{K} \sum_{\substack{i,j \in \mathbf O_k\\i,j \ne o}} \mathbf S_{ij} - \sum_{\substack{i \in \mathbf O_l\\ i \ne o}} \mathbf (\mathbf S_{io} + \mathbf S_{oi}) - \mathbf S_{oo} \, .
\end{equation}
Similarly, the cost $R^{SMC}(\mathbf c_{o \to l'},\mathbf S), l' \ne l$ is obtained by
\begin{align}
&R^{SMC}(\mathbf c_{o \to l'},\mathbf S) = - \sum_{k=1}^{K} \sum_{\substack{i,j \in \mathbf O_k\\i,j \ne o}} \mathbf S_{ij} - \sum_{\substack{i \in \mathbf O_{l'}\\ i \ne o}} \mathbf (\mathbf S_{io} + \mathbf S_{oi}) - \mathbf S_{oo}  \nonumber \\
&= R^{SMC}(\mathbf c_{o \to l},\mathbf S)  + \sum_{\substack{i \in \mathbf O_l\\ i \ne o}} \mathbf (\mathbf S_{io} + \mathbf S_{oi}) - \sum_{\substack{i \in \mathbf O_{l'}\\ i \ne o}} \mathbf (\mathbf S_{io} + \mathbf S_{oi}) \,.
\label{eq:efficientOpt}
\end{align}
Thus, given $R^{SMC}(\mathbf c_{o \to l},\mathbf S)$  the runtime of a new evaluation of the cost function $R^{SMC}(\mathbf c_{o \to l'},\mathbf S)$  is $\mathcal O(n)$. Hence, the total runtime of the local search method will be $\mathcal O(tn^2)$.
Therefore, at the beginning, we compute a random initial solution, wherein each object is assigned randomly to one of $K$ clusters,  and compute the respective cost. At each iteration, we use Eq.~\ref{eq:efficientOpt} to investigate the cost of assigning an object to the other clusters than the current one. Then, we assign the object to the cluster that yields a maximal reduction in the cost.
We might repeat the local search algorithm with several random initializations and at end, choose a solution with a minimal cost. Note that even the efficient evaluation and optimization of the variants in Eq.~\ref{eq:minCutKway} and  Eq.~\ref{eq:CC} would yield  $\mathcal O(tKn^2)$ total runtime, i.e.,  $K$ times slower than the variant expressed in Eq.\ref{eq:sMaxCor}.

We note that this technique can be employed with other optimization or inference methods as well, such as MCMC methods and simulated annealing.

\paragraph{On the convergence rate of the local search optimization.} With the co-authors,  we have shown in \cite{ThielCD19} that for  \emph{Correlation Clustering}, Frank-Wolfe optimization with line search for the update parameter (to find the optimal learning rate) is equivalent to the local search algorithm. On the other hand, we have established convergence rate of $\mathcal O(\frac{1}{t})$ for Frank-Wolfe optimization applied to \emph{Correlation Clustering}  \cite{ThielCD19} ($t$ indicates the optimization step). As discussed before, given the shifted pairwise similarities, \emph{Shifted Min Cut} is equivalent to \emph{Correlation Clustering}. Thus, the same argument holds for the aforementioned local search algorithm for \emph{Shifted Min Cut}, i.e., \emph{Shifted Min Cut} enjoys the convergence rate of $\mathcal O(\frac{1}{t})$.
This convergence rate should be compared with the convergence rate of  $\mathcal O(\frac{1}{\sqrt t})$ for general non-convex (non-concave) functions \cite{Reddi16} that applies to many other clustering objectives such \emph{Ratio Assoc}, \emph{Normalized Cut} and \emph{Dominant Set Clustering}, i.e., optimizing \emph{Shifted Min Cut} yields a faster theoretical convergence rate compared to many other alternatives.

\section{Shift Analysis of Other Clustering Methods}

In this section, we investigate the impact of shifting the pairwise similarities on some common flat and hierarchical clustering methods.

\paragraph{Shift of pairwise similarities for flat clustering.}
It is obvious that $K$-means and Gaussian Mixture Models (GMMs) are invariant with respect to the shift of data features. Since these methods perform directly on the data features, shifting refers to adding constant $\alpha$ to all the features. Under this shift, the centroids (in $K$-means) and the means (in GMM) are shifted by $\alpha$  as well, but their proportional distances stay the same.
The other parameters, i.e., the clustering assignments (in $K$-means), and the assignment probabilities, covariance matrices and weights (in GMM) do not change. One might assume that by shift only the location of the clusters is affected without modifying the cluster memberships.
A similar argument applies to a density-based clustering method such as DBSCAN \cite{DBSCAN} wherein shifting  data features does not modify the clustering solution, except a consistent shift of the geographical locations of the clusters together.

As discussed in \cite{roth03PAMI}, when shifting the pairwise similarities by $\alpha$, the \emph{Ratio Assoc} and \emph{Ratio Cut} cost functions stay invariant, i.e., their optimal solutions stay the same.
By shifting the pairwise similarities by  $\alpha$, the \emph{Ratio Assoc} cost function is written as

\begin{align}
    &R^{SRA}(\mathbf c, \mathbf{X},\alpha) = -\sum_{k=1}^{K} \sum_{i,j\in \mathbf O_k} \frac{\mathbf X_{ij}+\alpha}{\vert \mathbf O_k \vert} \nonumber \\
    &\qquad\qquad = {-\sum_{k=1}^{K} \sum_{i,j\in \mathbf O_k} \frac{\mathbf X_{ij}}{\vert \mathbf O_k \vert}} -\sum_{k=1}^{K} \sum_{i,j\in \mathbf O_k} \frac{\alpha}{\vert \mathbf O_k \vert} \nonumber \\
    &\qquad\qquad= {-\sum_{k=1}^{K} \sum_{i,j\in \mathbf O_k} \frac{\mathbf X_{ij}}{\vert \mathbf O_k \vert}} -\sum_{k=1}^{K} \frac{\alpha \vert \mathbf O_k \vert^2}{\vert \mathbf O_k \vert} \nonumber \\
    &\qquad\qquad= \underbrace{-\sum_{k=1}^{K} \sum_{i,j\in \mathbf O_k} \frac{\mathbf X_{ij}}{\vert \mathbf O_k \vert}}_{R^{RA}(\mathbf c, \mathbf{X})} -\underbrace{\alpha n}_{constant}.
\end{align}
Therefore, the \emph{Ratio Assoc} cost function is \emph{invariant} under shifting the pairwise similarities.
Similar to \emph{Ratio Assoc}, the \emph{Shifted Ratio Cut} cost function can be written as
\begin{align}
    &R^{SRC}(\mathbf c, \mathbf X, \alpha) = \sum_{k=1}^{K} \frac{\sum_{i\in \mathbf O_k}\sum_{j\in \mathbf O \setminus \mathbf O_k} \mathbf X_{ij}+\alpha} {|\mathbf O_k|} \nonumber \\
    & = \sum_{k=1}^{K} {\frac{\sum_{i\in \mathbf O_k}\sum_{j\in \mathbf O \setminus \mathbf O_k} \mathbf X_{ij}} {|\mathbf O_k|}} + \sum_{k=1}^{K} \frac{\sum_{i\in \mathbf O_k}\sum_{j\in \mathbf O \setminus \mathbf O_k}\alpha} {|\mathbf O_k|} \nonumber \\
    &  = \sum_{k=1}^{K} {\frac{\sum_{i\in \mathbf O_k}\sum_{j\in \mathbf O \setminus \mathbf O_k} \mathbf X_{ij}} {|\mathbf O_k|}} + \sum_{k=1}^{K} \frac{\alpha |\mathbf O_k| (n-|\mathbf O_k|)} {|\mathbf O_k|} \nonumber \\
    &  = \sum_{k=1}^{K} \underbrace{\frac{\sum_{i\in \mathbf O_k}\sum_{j\in \mathbf O \setminus \mathbf O_k}\mathbf X_{ij}} {|\mathbf O_k|}}_{R^{RC}(\mathbf c, \mathbf X)} + \underbrace{\alpha n (K-1)}_{constant}.
\end{align}

Thereby, both  \emph{Ratio Assoc} and \emph{Ratio Cut} cost functions are \emph{invariant} under shifting the pairwise similarities. One can show that this holds in general for every clustering cost function  that normalizes the clusters by the size of the clusters, i.e., size-normalized (divided) clustering cost functions stay \emph{invariant} with respect to the shift of pairwise similarities.

On the other hand, the \emph{Normalized Cut} cost function when the pairwise similarities are shifted is written by

\begin{equation}
    R^{SNC}(\mathbf c, \mathbf X,\alpha) = \sum_{k=1}^{K} \frac{\sum_{i\in \mathbf O_k}\sum_{j\in \mathbf O \setminus \mathbf O_k} \mathbf X_{ij}+\alpha} {\sum_{i \in \mathbf O_k}\sum_{j\in \mathbf O}\mathbf X_{ij}+\alpha}.
\end{equation}

It turns out that this cost function \emph{is not} shift invariant in general, contrary to the two previous alternatives. However, for the special case of almost balanced clusters, i.e., $|\mathbf O_k| \approx n/K,\,\; \forall 1\leq k \leq K$,\footnote{Similar to \emph{Shifted Min Cut}, $n/K$ might not be an integer number. Then, we consider $\left \lceil{n/K}\right \rceil$  and $\left \lfloor{n/K}\right \rfloor$ instead of $n/K$.} and similar intra-cluster similarity distribution among all clusters, all the row-sums of the similarity matrix $\mathbf X$ tend to be close to each other. The objects then share the same degree, i.e., $\sum_{j=1}^n \mathbf X_{ij} \approx  constant$. In this case, the \emph{Normalized Cut} cost functions becomes equivalent to the \emph{Ratio Assoc} cost function \cite{roth03PAMI}.
This analysis explains the similar performance of such graph partitioning methods in large-scale comparison studies, e.g., for image segmentation, where clusters have balanced and similar structures \cite{SoundararajanS01,roth03PAMI}.

\emph{Ratio Cut}, despite normalizing the cut by the size of clusters, intends to separate small clusters, as demonstrated in \cite{JS:JM:PAMI:2000,phdthesis_Morteza}. For this reason, \emph{Normalized Cut} has proposed to normalize the cut by the degree of the clusters, rather than the size of the clusters. An alternative way to overcome this problem is to apply a stronger constraint on the size of the clusters. Using this idea, \emph{P-Spectral Clustering} \cite{Buhler:2009PSpec} proposes a nonlinear generalization of spectral clustering based on the second eigenvector of the graph $p$-Laplacian which is then interpreted as a generalization of graph clustering models such as \emph{Ratio Cut}. \emph{P-Spectral Clustering} is an iterative  clustering procedure that at each step performs a bi-partitioning of one of the existing clusters until $K$ clusters are constructed using a nonlinear spectral method.
The underlying respective cost function for bi-partitioning  into two sets $\mathbf O_a$ and $\mathbf O_b$ is given by ($p >1$)

\begin{eqnarray}
    R^{PSC}(c,\mathbf X) = \sum_{i\in\mathbf O_a}\sum_{j\in\mathbf O_b} \mathbf X_{ij}\left( \frac{1}{|\mathbf O_a|^{\frac{1}{p-1}}} + \frac{1}{|\mathbf O_b|^{\frac{1}{p-1}}} \right)^{p-1}.
\end{eqnarray}

In \cite{phdthesis_Morteza}, we have introduced \emph{Adaptive Ratio Cut} (ARC) as a generalization of the  cost function to yield $K$  clusters:

\begin{eqnarray}
    R^{ARC}(c,\mathbf{X}) &=& \sum_{k=1}^K \sum_{k'=k+1}^K \sum_{i\in\mathbf O_k}\sum_{j\in\mathbf O_{k'}} \mathbf X_{ij} \left(\frac{1}{|\mathbf O_k|^{\frac{1}{p-1}}}+ \frac{1}{|\mathbf  O_{k'}|^{\frac{1}{p-1}}} \right)^{p-1}.
\end{eqnarray}

For the special case of $p=2$,  \emph{Adaptive Ratio Cut} is equivalent to the standard \emph{Ratio Cut} cost function.
However, unlike \emph{Ratio Cut}, it is easy to see that \emph{Adaptive Ratio Cut} is not shift invariant, as the shift parameter $\alpha$ cannot be factored out from the cost function.

\paragraph{Shifted Dominant Set Clustering.}
This clustering method computes the clusters via performing replicator dynamics. It has been shown that the solutions of a replicator dynamics correspond to the solutions of the following quadratic program \cite{SchusterRD83,GVK97}.
\begin{equation}
\max_\mathbf v \; f(\mathbf v) = \mathbf v^{\texttt{T}} \mathbf X \mathbf v, \;\; \texttt{s.t.} \; \mathbf v \ge \mathbf 0\, , \sum_{i=1}^{n}\mathbf v_i = 1\, ,
\label{eq:quadProg}
\end{equation}
where the \emph{n}-dimensional \emph{characteristic vector} $\mathbf v$ determines the participation of the objects to the solution.

Thus, to study the impact of the shift on DSC, we consider the shifted variant of the quadratic program. In \cite{mlChehreghani16} we have elaborated the impact of such a shift based on the off-diagonal shift argument in   \cite{PavanPHierarchy03}. It yields
\begin{eqnarray}
    f(\mathbf v, \alpha) &=& \mathbf v^{\texttt{T}} (\mathbf X + \alpha\, \mathbf e \mathbf e^{\texttt{T}}) \mathbf v \nonumber\\
      &=& \mathbf v^{\texttt{T}} \mathbf X \mathbf v + \mathbf v^{\texttt{T}}  \alpha\, \mathbf e \mathbf e^{\texttt{T}} \mathbf v  \nonumber\\
    &=& \mathbf v^{\texttt{T}} \mathbf X \mathbf v + \alpha\, \underbrace{(\mathbf v^{\texttt{T}}  \mathbf e)}_{=1} \, \underbrace{(\mathbf e^{\texttt{T}} \mathbf v)}_{=1}  \nonumber\\
    &=& \mathbf v^{\texttt{T}} \mathbf X \mathbf v + \alpha \, ,
\end{eqnarray}
where $\mathbf e = (1,1,...1)^{\texttt{T}}$ is a vector of ones.

Therefore, \emph{Dominant Set Clustering} is \emph{invariant} under shifting the pairwise similarities.

However, it has been proposed in \cite{PavanPHierarchy03} to shift the diagonal entries of the similarity matrix by a negative value, in order to obtain coarser clusters, which yields computing a hierarchy of clusters. The clusters obtained from the unshifted similarity matrix appear at the lowest level of the hierarchy. The larger the negative shift is the coarser the clusters are. Performing a negative shift is equivalent to adding the same shift but with a positive sign to the off-diagonal pairwise similarities. Thereby, the shifted matrix is still non-negative and has a null diagonal, i.e. satisfies the conditions of \emph{Dominant Set Clustering}.

One can think of performing a negative shift on the off-diagonal pairwise similarities to compute a finer representation of the clusters. However, this type of shift might violate the non-negativity and null diagonal constraints. On the other hand, according to our experiments, a negative shift is effectively equivalent to applying a larger cut-off threshold when peeling off the clusters. In \cite{mlChehreghani16} we have proposed such a shift to accelerate the appearance of clusters for DSC.

\paragraph{Shift of pairwise similarities for hierarchical clustering.}
Hierarchical clustering methods, unlike flat clustering, produce clusters at multiple levels. A main category of such methods  first consider each object in a separate cluster, and then at each step, combine the two clusters with a \emph{minimal} distance according to some criterion until only one cluster is left at the highest level.

A cluster at an arbitrary level  is represented by a set of objects belong to that, e.g., by  $\mathbf u$ or $\mathbf v$. A hierarchical  clustering  solution can be represented by a dendrogram (tree) $T$ such that,
\noindent i) each node $\mathbf v$ in $T$ consists of a non-empty subset of the objects that belong to cluster  $\mathbf v$, and
ii)  the overlapping clusters  have a parent-child relation, i.e., one is the (grand) parent of the other.

We use $dist(\mathbf u, \mathbf v)$ to refer to the inter-cluster distance between clusters $\mathbf u$ and $\mathbf v$. It can be defined according to different criteria. Three common criteria for hierarchical clustering are \emph{single} linkage, \emph{complete} linkage and \emph{average} linkage.
Given the matrix of (inter-object) pairwise dissimilarities $\mathbf D=\{\mathbf D_{ij}\}, i,j \in \mathbf O$, the \emph{single} linkage criterion \cite{sneath1957dn09j} defines the distance between  every two clusters as the distance between their nearest members:
\begin{equation}
  dist(\mathbf u,\mathbf v) = \min_{i \in \mathbf u, j \in \mathbf v} \mathbf D_{ij} \, .
\end{equation}

On the other hand, \emph{complete} linkage  \cite{lance67hierarchical} considers the distance between their farthest members:
\begin{equation}
  dist(\mathbf u,\mathbf v) = \max_{i \in \mathbf u, j \in \mathbf v} \mathbf D_{ij} \, .
\end{equation}

Finally,  \emph{average} linkage \cite{sokal58} uses the average of the inter-cluster distances  as the distance between the two clusters:

\begin{equation}
  dist(\mathbf u,\mathbf v) = \sum_{i \in \mathbf u, j \in \mathbf v} \frac{\mathbf D_{ij}}{|\mathbf u||\mathbf v|} \, .
\end{equation}

In the following we show that these methods, which perform based on pairwise inter-cluster distances, are shift-invariant (Proposition \ref{thm:agg_shift}). 

\begin{proposition}
 {Single} linkage, {complete} linkage and {average} linkage methods are invariant with respect to the shift of the pairwise dissimilarities $\mathbf D$ by constant $\alpha$.
  \label{thm:agg_shift}
\end{proposition}

\begin{proof}
Let us show the shifted pairwise dissimilarities by  $\mathbf D^\alpha$, i.e., $\mathbf D^\alpha_{ij} = \mathbf D_{ij} + \alpha, \forall i,j \in \mathbf O$.
\begin{itemize}
  \item With shifting all the pairwise dissimilarities by $\alpha$, the  $dist(\mathbf u,\mathbf v)$  function for \emph{single} linkage is defined as
  \begin{equation}\label{eq:single_linkage_shift}
    dist(\mathbf u,\mathbf v) = \min_{i \in \mathbf u, j \in \mathbf v} \mathbf D^\alpha_{ij} = \min_{i \in \mathbf u, j \in \mathbf v} \mathbf D_{ij} + \alpha.
  \end{equation}
  Thus, if $dist(\mathbf u,\mathbf v) \le dist(\mathbf u,\mathbf w)$ holds  with respect to   $\mathbf D$, then it would also hold with respect to $\mathbf D^\alpha$ and vice versa, as they differ only by a constant in both sides of the inequality. Thus, shifting the pairwise dissimilarities by $\alpha$ does not change the order of merging the intermediate clusters and hence the final dendrogram will remain the same.
  \item With shifting all the pairwise dissimilarities by $\alpha$, the  $dist(\mathbf u,\mathbf v)$  function for \emph{complete} linkage is defined as
  \begin{equation}\label{eq:complete_linkage_shift}
    dist(\mathbf u,\mathbf v) = \max_{i \in \mathbf u, j \in \mathbf v} \mathbf D^\alpha_{ij} = \max_{i \in \mathbf u, j \in \mathbf v} \mathbf D_{ij} + \alpha.
  \end{equation}
  Thus, with the same argument as with \emph{single} linkage, shifting the  pairwise dissimilarities by $\alpha$ does not change the final \emph{complete} linkage dendrogram.
  \item With shifting all the pairwise dissimilarities by $\alpha$, the  $dist(\mathbf u,\mathbf v)$  function in \emph{average} linkage is defined as
  \begin{eqnarray}\label{eq:average_linkage_shift}
    dist(\mathbf u,\mathbf v) &=& \sum_{i \in \mathbf u, j \in \mathbf v} \frac{\mathbf D^\alpha_{ij}}{|\mathbf u||\mathbf v|}  \nonumber \\
    &=& \sum_{i \in \mathbf u, j \in \mathbf v} \frac{\mathbf D_{ij} + \alpha}{|\mathbf u||\mathbf v|}   \nonumber \\
    &=&  \left(\sum_{i \in \mathbf u, j \in \mathbf v} \frac{\mathbf D_{ij}}{|\mathbf u||\mathbf v|}\right)  + \alpha .
  \end{eqnarray}
  Thus, we use the same argument as in with \emph{single} linkage and \emph{complete} linkage, and conclude that shifting the  pairwise dissimilarities by $\alpha$ does not change the final \emph{average} linkage dendrogram.
\end{itemize}
\hfill $\square$
\end{proof}

Another category of hierarchical clustering methods such as  \emph{centroid}  linkage and \emph{Ward} linkage perform directly on data features, instead of pairwise dissimilarities. \emph{Centroid}  linkage computes a representative for each cluster and defines the inter-cluster distances according to those representatives. Similar to the case of $K$-means, shifting the data features by a constant does not change the pairwise  inter-cluster distances. The \emph{Ward} linkage \cite{Inchoate:Ward63}  aims at minimizing  the within-cluster variance at each step, i.e., the $dist(\mathbf u,\mathbf v)$ is defined as

\begin{equation}
dist(\mathbf u,\mathbf v) = \frac{|\mathbf u||\mathbf v|}{|\mathbf u|+|\mathbf v|} ||\mathbf g_{\mathbf u} -  \mathbf g_{\mathbf v}||^2 \, ,
\end{equation}
where $ \mathbf g_{\mathbf u}$ denotes the centroid  vector of  cluster $\mathbf u$. Therefore, due to shift invariance of variance, the \emph{Ward} linkage is also invariant with respect to the shift of data features.
Thereby, we can state Proposition \ref{thm:agg_shift_centroid_Ward} as following.

\begin{proposition}
 {Centroid}  linkage and {Ward} linkage  are invariant with respect to the shift of data features.
  \label{thm:agg_shift_centroid_Ward}
\end{proposition}

Finally, it is notable that  some of the improvements proposed for hierarchical clustering still preserve the invariance property with respect to the shift of pairwise distances. For example, in order to improve the robustness of hierarchical clustering, it is suggested in \cite{ChehreghaniAC08} to first apply $K$-means with many centroids (of order of $n$) and then apply the aforementioned hierarchical methods. Since both steps, i.e., $K$-means clustering and hierarchical clustering, are invariant with respect to the shift, thus one can conclude that the entire procedure remains invariant as well.
The work in \cite{ChehreghaniIJCNN2021} studies extracting all mutual linkages at every step of hierarchical clustering, instead of the smallest one, in order to provide adaptivity to diverse shapes of clusters. Since this contribution is independent of the way the inter-cluster distances are defined, then this strategy yields invariant clustering with respect to the shift of pairwise distances for methods such as \emph{single} linkage, \emph{complete} linkage and \emph{average} linkage.

\section{Experiments}
\label{sec:experiments}

We empirically  investigate the performance of \emph{Shifted Min Cut} and compare the results against several alternatives.  We perform the experiments under identical computational settings on a  core i7-4600U Intel machine with 2.7 GHz CPU and  8.00 GB internal memory.

\paragraph{Data.}

We first perform our experiments on several UCI datasets~\cite{Lichman:2013}, chosen from different domains and contexts with different type of features.
\begin{enumerate}
\item \emph{Breast Tissue}: contains $106$ electrical impedance measurements of the breast tissue samples in $6$ types (clusters) each with $10$ features. The types or clusters are `car' (carcinoma, $21$ measurements), `fad' (fibro-adenoma, $15$ measurements), `mas' (mastopathy $18$ measurements), `gla' (glandular, $16$ measurements), `con' (connective, $12$ measurements) and `adi' (adipose $22$ measurements). The features are real valued with no missing value.
\item \emph{Cloud}: consists of $2048$  vectors, where each vector includes $10$ parameters in two types (each of size $1024$) representing AVHRR images. The vectors (attributes) are real-valued and there are no missing values. The target clusters are balanced.
\item \emph{Ecoli}: a biological dataset on the cellular localization sites of $7$ types (clusters) of proteins which includes $336$ samples. The samples are represented by $8$ real-valued features. The size of the clusters are: $143, 77, 3, 7, 35, 20$ and $52$,
\item \emph{Forest Type Mapping}: a remote sensing dataset of $523$ samples with $27$ real-valued attributes collected from forests in Japan and grouped in $4$ different forest types (clusters). The clusters are:  `s' (`Sugi' forest, $159$ samples), `h' (`Hinoki' forest, $86$ samples), `d' (`Mixed deciduous' forest, $195$ samples), `o' (`Other' non-forest land, $83$ samples).
\item \emph{Heart}: dataset of heart disease that involves $303$ instances each with $75$ attributes. The attributes are diverse: categorical, integer and real where the categorical attributes are treated using one-hot encoding. The missing values are estimated by the median of the respective feature. Cluster distributions are: $164,  55,  36,  35$ and  $13$.
\item \emph{Lung Cancer}: high-dimensional lung cancer data with $32$ instances (with distribution $9$ and $23$) and $56$ integer features. There are few missing values estimated using the median of the respective feature.
\item \emph{Parkinsons}: contains $197$ biomedical voice measurements from $31$ people each represented by $23$ real-valued attributes that correspond to voice recordings. In the dataset, there are $48$ healthy  samples and $147$ other samples that belong to one of $23$ people with  Parkinson's disease.
\item  \emph{Pima Indians Diabetes}:   the data of $768$ female patients from Pima Indian heritage with $8$ attributes. The attributes  include the number of pregnancies of the patient, their BMI, insulin level, age, and so on, and they are either real numbers or integers. $268$ samples out of $768$ haze the outcome $1$ and the others ($500$ samples) have the outcome $0$.
\item \emph{SPECTF}: describes diagnosing of cardiac Single Proton Emission Computed Tomography (SPECT)  images with $44$  integer attributes (values from the $0$ to $100$) about the heart of  $267$ patients. The diagnosis is binary with the distribution of $55$ and $212$ samples.
\item \emph{Statlog ACA (Australian Credit Approval)}: contains information of $690$ credit card applications each described with $14$ features (with cluster size $383$ and $307$). The features are categorical and numerical where for categorical features we use one-hot encoding. The few missing values are estimated using the median of the respective feature.
\item \emph{Teaching Assistant}: consists of evaluations of teaching performance over $5$ semesters of $151$ teaching assistant assignments. The scores are divided into $3$ roughly equal-sized categories (`low', `medium' and `high') to form the target variables which are used as the cluster labels. The attributes are categorical and integer, where we use one-hot encoding for categorical attributes. There are no missing values.
\item \emph{User Knowledge Modeling}: contains the $403$ students' knowledge status on Electrical DC Machines with $5$ integer attributes  grouped in $4$ categories.  The labels and the cluster distributions are: `very Low': $50$, `low': $129$, `middle': $122$ and `high': $130$.  There are no missing values.
\end{enumerate}

In these datasets, the objects  are represented by vectors.  Thus, to obtain the pairwise similarity matrix $\mathbf X$, we first compute the pairwise squared Euclidean distances between the  vectors and obtain matrix $\mathbf D$. Then, as proposed in~\cite{mlChehreghani16}, we convert the pairwise distances $\mathbf D$ to the similarity matrix $\mathbf X$ via $\mathbf X_{ij} = \max(\mathbf D) - \mathbf D_{ij} + \min(\mathbf D)$, where the $\max(.)$ and $\min(.)$ operations respectively give the maximum and the minimum of the elements in $\mathbf D$. An alternative transformation is an exponential function in the form of $\mathbf S_{ij} = \exp(-\frac{\mathbf X_{ij}}{\sigma^2})$, which requires fixing the free parameter $\sigma$ in advance. However, this task is nontrivial in unsupervised learning and the appropriate values of $\sigma$ coincide in a very narrow range~\cite{Luxburg:2007}. The other alternative is the cosine similarity, which suits better to textual and document datasets. On our datasets, we consistently obtain better results with the aforementioned transformation.

\begin{table*}
\caption{Performance of different methods with respect to the adjusted Mutual Information criterion. \emph{Shifted Min Cut} yields the best results in most of the cases.}  
\centering  
\begin{tabular}{c | c c c c c c c c}  
\hline\hline  
dataset & ShiftedMinCut & DSC & InImDyn & PSC & GMM & $K$-means & PIC & SP \\ [0.5ex]  
\hline  
\emph{Breast Tissue} & 0.4196 & 0.4305 & 0.4196 & 0.3606 & 0.3276 & 0.1809 & 0.4123 & \textbf{0.4507} \\  
\emph{Cloud} & \textbf{1.0000} & 0.3812 & 0.3543 & 0.3098 & 0.8511 & 0.3056 & 0.8597 & 0.8406 \\  
\emph{Ecoli} & 0.5414 & 0.4731 & 0.4368 & 0.5074 & \textbf{0.5800} & 0.5685 & 0.0542 & 0.4743 \\  
\emph{Forest Type} & 0.4352 & 0.2960 & 0.3109  & 0.2704 & 0.3877 & \textbf{0.5197} & 0.3875 & 0.3163 \\  
\emph{Heart} & \textbf{0.1570} & 0.0698 & 0.0602 & 0.0594 & 0.0813 & 0.0813 & 0.0509 & 0.1078 \\  
\emph{Lung Cancer} & 0.2362 & 0.0850 & 0.0859 & 0.0713 & 0.1684 & 0.1997 & 0.0380 & \textbf{0.2473} \\  
\emph{Parkinsons} & \textbf{0.1957} & 0.0738 & 0.0761 & 0.0511 & 0.0484 & 0.0136 & 0.0153 & 0.1631 \\  
\emph{Pima Indians Diabetes} & 0.1178 & 0.0561 & 0.0533 & 0.0368 & 0.0003 & 0.0257 & \textbf{0.1226} & 0.1200 \\  
\emph{SPECTF} & \textbf{0.1570} & 0.0698 & 0.0602 & 0.0419 & 0.0813 & 0.0813 & 0.0509 & 0.1078 \\ 
\emph{Statlog ACA} & \textbf{0.3907} & 0.1607 & 0.1498 & 0.1392 & 0.0038 & 0.0038 & 0.3570 & 0.3683 \\ 
\emph{Teaching Assistant} & \textbf{0.1041} & 0.0268 & 0.0357 & 0.0123 & 0.0353 & 0.0130 & 0.0470 & 0.0123 \\ 
\emph{User Knowledge Modeling} & 0.2926 & 0.1107 & 0.1198 & 0.0441 & \textbf{0.6100} & 0.2139 & 0.2454 & 0.1169 \\ 
\hline 
\end{tabular}
\label{table:RealWorld_Mutual_info} 
\end{table*}

\begin{table*}
\caption{Performance of different methods with respect to the adjusted Rand score. \emph{Shifted Min Cut} leads to better clustering solutions on most of the datasets.}
\centering 
\begin{tabular}{c | c c c c c c c c} 
\hline\hline
dataset & ShiftedMinCut & DSC & InImDyn & PSC & GMM & $K$-means & PIC & SP \\ [0.5ex] 
\hline 
\emph{Breast Tissue} & \textbf{0.3546} & 0.2929 & 0.2907 & 0.3100 & 0.2085 & 0.0943 & 0.3125 & 0.3037 \\ 
\emph{Cloud} & \textbf{1.0000} & 0.3573 & 0.3117 & 0.2926 & 0.8991 & 0.2429 & 0.9065 & 0.8899 \\ 
\emph{Ecoli} & \textbf{0.6801} & 0.4068 & 0.3299 & 0.5145 & 0.5574 & 0.4944 & 0.0378 & 0.4132 \\ 
\emph{Forest Type} & 0.3983 & 0.2225 & 0.2426 & 0.2027 & 0.3285 & \textbf{0.4987} & 0.3560 & 0.2454 \\ 
\emph{Heart} & \textbf{0.0917} & 0.0608 & 0.0449 & 0.0578 & 0.0467 & 0.0337 & 0.0721 & 0.0588 \\ 
\emph{Lung Cancer} & 0.2962 & 0.0689 & 0.0664 & 0.0215 & 0.1698 & 0.2294 & 0.0949 & \textbf{0.4201} \\ 
\emph{Parkinsons} & \textbf{0.1275} & 0.0129 & 0.0360 & 0.0234 & 0.0123 & 0.0162 & 0.0178 & 0.0419 \\ 
\emph{Pima Indians Diabetes} & 0.1535 & 0.0454 & 0.0460 & 0.0381 & 0.0010 & 0.0744 & \textbf{0.1617} & 0.1200 \\ 
\emph{SPECTF} & \textbf{0.0917} & 0.0608 & 0.0429 & 0.0570 & 0.0617 & 0.0617 & 0.0434 & 0.0898 \\ 
\emph{Statlog ACA} & \textbf{0.4913} & 0.1485 & 0.1307 & 0.1332 & 0.0022 & 0.0022 & 0.4550 & 0.4710 \\ 
\emph{Teaching Assistant} & \textbf{0.1170} & 0.0112 & 0.0127 & 0.0129 & 0.0220 & 0.0089 & 0.0322 & 0.0129 \\ 
\emph{User Knowledge Modeling} & 0.2912 & 0.0713 & 0.0778 & 0.0503 & \textbf{0.5680} & 0.1675 & 0.1449 & 0.1053 \\ 
\hline 
\end{tabular}
\label{table:RealWorld_Rand}
\end{table*}

\begin{table*}[t]
\caption{Performance of different methods with respect to the adjusted V-measure. In a consistent way to the two previous evaluation criteria, the \emph{Shifted Min Cut} method provides the best clustering results on most of the datasets.} 
\centering 
\begin{tabular}{c | c c c c c c c c} 
\hline\hline 
dataset & ShiftedMinCut & DSC & InImDyn & PSC & GMM & $K$-means & PIC & SP \\ [0.5ex]
\hline 
\emph{Breast Tissue} & \textbf{0.5563} & 0.4914 & 0.4999 & 0.4756 & 0.4389 & 0.2895 & 0.5230 & 0.5142 \\ 
\emph{Cloud} & \textbf{1.0000} & 0.5525 & 0.5239 & 0.5116 & 0.8520 & 0.3388 & 0.8605 & 0.8417 \\ 
\emph{Ecoli} & \textbf{0.6396} & 0.5339 & 0.5169 & 0.5203 & 0.6192 & 0.6321 & 0.1139 & 0.5364 \\ 
\emph{Forest Type} & 0.4835 & 0.3591 & 0.3866 & 0.2729 & 0.3937 & \textbf{0.5279} & 0.3982 & 0.3244 \\ 
\emph{Heart} & \textbf{0.1788} & 0.1186 & 0.1061 & 0.0689 & 0.0896 & 0.0896 & 0.0644 & 0.1131 \\ 
\emph{Lung Cancer} & 0.2730 & 0.2175 & 0.2310 & 0.1038 & 0.2030 & 0.2356 & 0.1117 & \textbf{0.2987} \\ 
\emph{Parkinsons} & \textbf{0.2196} & 0.1255 & 0.1327 & 0.0814 & 0.0130 & 0.0105 & 0.0291 &0.1798 \\ 
\emph{Pima Indians Diabetes} & 0.1227 & 0.0867 & 0.0838 & 0.0693 & 0.0013 & 0.0295 & \textbf{0.1276} & 0.1200 \\ 
\emph{SPECTF} & \textbf{0.1788} & 0.1186 & 0.1061 & 0.0819 & 0.0896 & 0.0896 & 0.0644 & 0.1131 \\ 
\emph{Statlog ACA} & \textbf{0.3927} & 0.2396 & 0.2267 & 0.2068 & 0.0099 & 0.0099 & 0.3632 & 0.3720 \\ 
\emph{Teaching Assistant} & \textbf{0.1156} & 0.0615 & 0.0770 & 0.0253 & 0.0583 & 0.0263 & 0.0606 & 0.0253 \\ 
\emph{User Knowledge Modeling} & 0.3384 & 0.1527 & 0.1697 & 0.0573 & \textbf{0.6191} & 0.2278 & 0.2603 & 0.1325 \\ 
\hline 
\end{tabular}
\label{table:RealWorld_V_Measure} 
\end{table*}

\paragraph{Methods.}

We compare \emph{Shifted Min Cut} against several alternative methods developed for clustering. We consider the following methods:
 i) \emph{Dominant Set Clustering} (DSC),
 ii) \emph{InImDyn},
 iii) \emph{P-Spectral Clustering} (PSC),
iv) \emph{Gaussian Mixture Model} (GMM),
v) \emph{$K$-means},
vi) \emph{Power Iteration Clustering} (PIC), and
vii) \emph{Spectral Clustering} (SC).

The chosen baselines belong to different clustering approaches which  cover a wide range of alternative viewpoints for clustering, e.g., those based on a cost function, probabilistic methods, game-theoretic methods and spectral methods. With the GMM method, we obtain the probabilistic assignment of the objects to the clusters.  Then, we assign each object to the most probable cluster.
 The developed clustering perspective can  potentially  be combined with the recent developments proposed in particular for cost-based clustering methods. For example, a category of recent clustering methods aim to combine deep representation learning methods with clustering \cite{abs-2007-14524,YangCL019}, or develop approximate and distributed clustering methods. Such  contributions are orthogonal to our contribution and, in principle,  can be combined with \emph{Shifted Min Cut} as well. On the other hand, considering the relation between \emph{Shifted Min Cut} and \emph{Correlation Clustering}, with the co-authors, we have recently \cite{ThielCD19} studied the performance of the local search optimization compared to a wide range of approximate methods developed for \emph{Correlation Clustering} and have demonstrated both efficiency and effectiveness for the local search method.

\paragraph{Evaluation criteria.}
We have access to the ground truth solutions for the datasets. These labels may play the role of an expert (reference) that tells us the desired clustering solution. Thus, we can use them to evaluate the results of different methods.  We note that we do not employ them to infer the clustering solution, they are only used for evaluation. Therefore, we are still in  unsupervised setting which assumes no data label is used to obtain the results. This evaluation procedure is recommended in \cite{manning2008introduction}  consistent with several studies, e.g., \cite{Dhillon:2004,LinC10PIC,LiuLY13,ThielCD19,YangCL019}.
Thereby, we  compare the true (given) clustering labels and the estimated  solutions to investigate quantitatively the performance of each method. For this purpose, we consider three criteria:
\begin{enumerate}
  \item adjusted Mutual Information~\cite{Vinh:2010}: the mutual information between the two estimated and true clustering solutions,
  \item adjusted Rand score~\cite{hubert1985comparing}: the similarity between the solutions,  and
  \item V-measure~\cite{RosenbergH07}: the harmonic mean of homogeneity and completeness.
\end{enumerate}

We compute the adjusted variant of these criteria such that they give zero scores for random solutions.

\paragraph{Results.}

We study the performance of different methods from two perspectives in order to distinguish between the quality of a method/costs function and its optimization. The former implies how good a particular method/cost function is (given that it can be optimized in a proper way) while the later focuses on the optimization aspects of the method/cost function. We run each method $100$ times with different random initializations.
In the first type of study, we choose  the best solution in terms of the cost or likelihood among the $100$ different runs for each method. We note that we do not choose the best results in terms of the evaluation criteria. This helps to gain a sense that the optimization is done properly, and we may not suffer from very poor local optima, and therefore, we can investigate the performance of the method or the cost function regardless of its optimization.

Tables \ref{table:RealWorld_Mutual_info}, \ref{table:RealWorld_Rand} and \ref{table:RealWorld_V_Measure} show the results of the first type of study for different clustering methods on the UCI datasets respectively with respect to the Mutual Information criterion, the Rand score and the V-measure. We observe that on most of the datasets, \emph{Shifted Min Cut} yields the best scores. In the cases that the method is not the best, it is usually among top choices. DSC and InImDyn perform very similarly, consistent to the results in~\cite{BuloPB11}. PIC works well only when there are few clusters in the dataset. The reason is that it computes an one-dimensional embedding of the data and then applies $K$-means. However, such an embedding might confuse some clusters when there exist many of them in the dataset \cite{mlChehreghani16}. PSC is significantly slower than the other methods and also yields suboptimal results, as reported by several previous studies as well. Other methods are efficient and perform within few seconds.

\begin{table*}
\begin{adjustwidth}{-1cm}{-1cm}
\caption{Average performance (and the standard deviation shown in brackets) for  different methods over $100$ runs with respect to  adjusted Mutual Information, where \emph{Shifted Min Cut } often yields the most promising results.}  
\centering  
\begin{tabular}{c | c c c c c}  
\hline\hline  
dataset & ShiftedMinCut  & GMM & $K$-means & PIC & SP \\ [0.5ex]  
\hline  
\emph{Breast Tissue} & $0.4191(0.0031)$  & $0.3031(0.0365)$ & $0.1794(0.0371)$ & $0.3883(0.0470)$ & $\mathbf{0.4393(0.0129)}$   \\  
\emph{Cloud} & $\mathbf{1.0000(0.0000)}$ & $0.8510(0.0152)$ & $0.3019(0.0141)$ & $0.8588(0.0060)$ & $0.8392(0.0071)$ \\  
\emph{Ecoli} & $0.5304(0.0119)$ & $\mathbf{0.5792(0.0374)}$ & 0.5673(0.0294) & $0.0493(0.0177)$ & $0.4728(0.0032)$ \\  
\emph{Forest Type} & $0.4085(0.0416)$ & $0.3542(0.0290)$ & $\mathbf{0.5169(0.0330)}$ & $0.3881(0.0147)$ & $0.3124(0.0010)$ \\  
\emph{Heart} & $\mathbf{0.1562(0.0024)}$ & $0.0817(0.0025)$ & $0.0832(0.0019)$ & $0.0528(0.0133)$ & $0.0820(0.0304)$ \\  
\emph{Lung Cancer} & $0.2174(0.0684)$ & $0.1678(0.0719)$ & $0.1720(0.0455)$ & $0.0460(0.0397)$ & $\mathbf{0.2389(0.0250)}$ \\  
\emph{Parkinsons} & $\mathbf{0.1932(0.0046)}$ & $0.0380(0.0241)$ & $0.0256(0.0305)$ & $0.0169(0.0134)$ & $0.1443(0.0251)$\\  
\emph{Pima Indians Diabetes} & $0.0972(0.0329)$ & $0.0003(0.0008)$ & $0.0285(0.0020)$ & $\mathbf{0.1233(0.0042)}$ & $0.1186(0.0049)$ \\  
\emph{SPECTF} & $\mathbf{0.1561(0.0026)}$ & $0.0821(0.0020)$ & $0.0806(0.0031)$ & $0.0503(0.0155)$ & $0.1019(0.0229)$ \\  
\emph{Statlog ACA} & $\mathbf{0.3842(0.0279)}$ & $0.0116(0.0031)$ & $0.0087(0.0033)$ & $0.3521(0.0045)$ & $0.3648(0.0037)$ \\  
\emph{Teaching Assistant} & $\mathbf{0.0934(0.0380)}$ & $0.0346(0.0123)$ & $0.0045(0.0085)$ & $0.0412(0.0124)$ & $0.0106(0.0073)$ \\  
\emph{User Knowledge Modeling} & $0.2735(0.0601)$ & $\mathbf{0.5536(0.0962)}$ & $0.2026(0.0648)$ & $0.2470(0.0131)$ & $0.1043(0.0239)$ \\  
\hline  
\end{tabular}
\label{table:RealWorld_Mutual_info_detail}  
\end{adjustwidth}
\end{table*}

\begin{table*}
\begin{adjustwidth}{-1cm}{-1cm}
\caption{Average performance (and the standard deviation) for different methods with respect to  adjusted Rand score. Consistent with adjusted Mutual Information, \emph{Shifted Min Cut } yields the best results on most of the datasets .}  
\centering  
\begin{tabular}{c | c c c c c}  
\hline\hline  
dataset & ShiftedMinCut  & GMM & $K$-means & PIC & SP \\ [0.5ex]  
\hline  
\emph{Breast Tissue} & $\mathbf{0.3543(0.0051)}$ & $0.1756(0.0354)$ & $0.1071(0.0361)$ & $0.2642(0.0697)$ & $0.3030(0.0311)$   \\  
\emph{Cloud} & $\mathbf{1.0000(0.0000)}$ & $0.8636(0.0309)$ & $0.2419(0.0015)$ & $0.9067(0.0018)$ & $0.8865(0.0029)$ \\  
\emph{Ecoli} & $\mathbf{0.6826(0.0113)}$ & $0.5303(0.0611)$ & $0.4598(0.0686)$ & $0.0193(0.0175)$ & $0.3939(0.0098)$ \\  
\emph{Forest Type} & $0.3894(0.0267)$ & $0.2792(0.0427)$ & $\mathbf{0.4880(0.0360)}$ & $0.3429(0.0243)$ & $0.2388(0.0009)$ \\  
\emph{Heart} & $\mathbf{0.0885(0.0041)}$ & $0.0471(0.0022)$ & $0.0259(0.0029)$ & $0.0518(0.0212)$ & $0.0498(0.0113)$ \\  
\emph{Lung Cancer} & $0.2546(0.0798)$ & $0.1594(0.0949)$ & $0.2079(0.0705)$ & $0.0844(0.0297)$ & $\mathbf{0.4172(0.0133)}$ \\  
\emph{Parkinsons} & $\mathbf{0.1268(0.0079)}$ & $0.0282(0.0286)$ & $0.0217(0.0389)$ & $0.0159(0.0121)$ & $0.0387(0.0081)$ \\  
\emph{Pima Indians Diabetes} & $0.1317(0.0329)$ & $0.0010(0.0003)$ & $0.0736(0.0027)$ & $\mathbf{0.1593(0.0028)}$ & $0.1135(0.0058)$ \\  
\emph{SPECTF} & $\mathbf{0.0903(0.0036)}$ & $0.0598(0.0033)$ & $0.0608(0.0042)$ & $0.0516(0.0248)$ & $0.0894(0.057)$ \\  
\emph{Statlog ACA} & $\mathbf{0.4825(0.0485)}$ & $0.0031(0.0007)$ & $0.0030(0.0008)$ & $0.4525(0.0042)$ & $0.4590(0.0162)$ \\  
\emph{Teaching Assistant} & $\mathbf{0.0883(0.0421)}$ & $0.0248(0.0061)$ & $0.0024(0.0069)$ & $0.0248(0.0115)$ & $0.0201(0.0097)$ \\ 
\emph{User Knowledge Modeling} & $0.2729(0.0520)$ & $\mathbf{0.4972(0.0788)}$ & $0.1542(0.0377)$ & $0.1495(0.0121)$ & $0.1046(0.0174)$ \\  
\hline
\end{tabular}
\label{table:RealWorld_Rand_detail} 
\end{adjustwidth}
\end{table*}

\begin{table*}
\begin{adjustwidth}{-1cm}{-1cm}
\caption{Average performance (and the standard deviation) for different methods with respect to  adjusted V-measure. We observe that the average results are consistent with the results from the first type of study and \emph{Shifted Min Cut} performs well compared to the alternatives. } 
\centering
\begin{tabular}{c | c c c c c}  
\hline\hline  
dataset & ShiftedMinCut  & GMM & $K$-means & PIC & SP \\ [0.5ex] 
\hline 
\emph{Breast Tissue} & $\mathbf{0.5558(0.0023)}$ & $0.4015(0.0326)$  & $0.2711(0.0329)$ & $0.5193(0.0190)$ & $0.5113(0.0174)$  \\  
\emph{Cloud} & $\mathbf{1.0000(0.0000)}$ & $0.8329(0.0137)$ & $0.3402(0.0011)$ & $0.8602(0.0011)$ & $0.8395(0.0028)$ \\  
\emph{Ecoli} & $\mathbf{0.6380(0.0149)}$ & $0.6029(0.0355)$ & $0.6007(0.0432)$ & $0.0884(0.0372)$ & $0.5296(0.0031)$ \\  
\emph{Forest Type} & $0.4712(0.0402)$ & $0.3835(0.0178)$ & $\mathbf{0.5211(0.0327)}$ & $0.3936(0.0147)$ & $0.3169(0.0017)$ \\  
\emph{Heart} & $\mathbf{0.1761(0.0037)}$ & $0.0841(0.0053)$ & $0.0864(0.0019)$ & $0.0608(0.0142)$ & $0.1066(0.0036)$ \\  
\emph{Lung Cancer} & $0.2578(0.0667)$ & $0.2004(0.0559)$ & $0.2194(0.0623)$ & $0.1094(0.0350)$ & $\mathbf{0.2858(0.0096)}$ \\  
\emph{Parkinsons} & $\mathbf{0.2151(0.0036)}$ & $0.0134(0.0102)$ & $0.0152(0.0096)$ & $0.0216(0.0142)$ & $0.1506(0.0243)$ \\  
\emph{Pima Indians Diabetes} & $0.1026(0.0328)$ & $0.0011(0.0006)$ & $0.0217(0.0058)$ & $\mathbf{0.1242(0.0042)}$ & $0.1205(0.0016)$ \\  
\emph{SPECTF} & $\mathbf{0.1729(0.0040)}$ & $0.0850(0.0028)$ & $0.0891(0.0015)$ & $0.0592(0.0058)$ & $0.0974(0.0108)$ \\  
\emph{Statlog ACA} & $\mathbf{0.3845(0.0463)}$ & $0.0040(0.0030)$ & $0.0033(0.0051)$ & $0.3627(0.0015)$ & $0.3655(0.0147)$ \\  
\emph{Teaching Assistant} & $\mathbf{0.1062(0.0375)}$ & $0.0506(0.0123)$ & $0.0170(0.0086)$ & $0.0535(0.0110)$ & $0.0284(0.0072)$ \\  
\emph{User Knowledge Modeling} & $0.3027(0.0592)$ & $\mathbf{0.5834(0.0674)}$ & $0.2250(0.0540)$ & $0.2612(0.0131)$ & $0.1257(0.0236)$ \\  
\hline  
\end{tabular}
\label{table:RealWorld_VI_detail}  
\end{adjustwidth}
\end{table*}

In the second type of study, in order to investigate the optimization itself, we report the average scores and the respective standard deviations over the $100$ different runs for each method. We note that DSC, InImDyn and PSC are non-randomized algorithmic procedures that do not show randomness in the performance and their results are stable among different runs. Therefore, we do not need to report their results here. Tables \ref{table:RealWorld_Mutual_info_detail}, \ref{table:RealWorld_Rand_detail} and  \ref{table:RealWorld_VI_detail} show such optimization variability results for different UCI datasets (i.e., the average results and the respective standard deviations shown in brackets). We observe that the results are consistent among different runs and the better methods in Tables \ref{table:RealWorld_Mutual_info},~\ref{table:RealWorld_Rand} and~\ref{table:RealWorld_V_Measure} perform well on average too, i.e., the results from the first type of study and the second type of study are consistent in overall. In particular, \emph{Shifted Min Cut} yields the most promising results in this type of study as well.  The results also confirm the effectiveness of the optimization based on local search, a method that is nowadays used widely in different machine learning paradigms.

\paragraph{Experiments on real-world data.}

In the following, we investigate the performance of different clustering methods on two real-world datasets:
\begin{enumerate}
  \item \emph{DS1}: This dataset,  collected by a document processing company, contains the vectors of $675$ scanned documents, wherein each document is represented in a $4096$ dimensional space using different textual, image, structural and other features. The documents are placed within $56$ clusters with different sizes, that makes the clustering task challenging. The size of the clusters varies from few documents to more than $200$ documents. The features are real-valued.
  \item \emph{DS2}: In this dataset, we collect articles about $5$ different Computer Science subjects: `artificial intelligence', `software', `hardware', `networks' and `algorithms'. For each category, we collect $1500$ articles, thus in total there are $7500$ articles in this dataset. We computer the \emph{tf-idf} vectors for each article, thus the attributes are numerical. There are no missing values.
\end{enumerate}

\begin{table*}
\caption{Performance of different methods on DS1. On this dataset, \emph{Shifted Min Cut} yields superior results compared to the alternatives.}  
\centering  
\begin{tabular}{c | c c c c c c c c}  
\hline\hline  
 & ShiftedMinCut & DSC & InImDyn & PSC & GMM & $K$-means & PIC & SP \\ [0.5ex]  
\hline 
\emph{Mutual Information} & 0.6046 & 0.2951 & 0.2786 & 0.1572 & \textbf{0.6124} & 0.5880 & 0.5994 & 0.3525  \\  
\emph{Rand score} & \textbf{0.4632} & 0.2101 & 0.2067 & 0.0762 & 0.2945 & 0.2434 & 0.2640 & 0.1384  \\  
\emph{V-measure} & \textbf{0.8107} & 0.4884 & 0.4655 & 0.3119 & 0.7996 & 0.7854 & 0.7733 & 0.5874  \\  
\hline  
\end{tabular}
\label{table:RealWorld_DS1}  
\end{table*}

\begin{table*}
\caption{Performance of different methods on DS2 where \emph{Shifted Min Cut} leads to the best overall performance.}
\centering  
\begin{tabular}{c | c c c c c c c c}  
\hline\hline  
 & ShiftedMinCut & DSC & InImDyn & PSC & GMM & $K$-means & PIC & SP \\ [0.5ex] 
\hline  
\emph{Mutual Information} & \textbf{0.5621} & 0.3011 & 0.3180 & 0.3572 & 0.4793 & 0.5103 & 0.4272 & 0.5475  \\  
\emph{Rand score} & 0.5316 & 0.2780 & 0.2621 & 0.2752 & 0.4648 & 0.4512 & 0.4484 & \textbf{0.5528}  \\  
\emph{V-measure} & \textbf{0.6179} & 0.3126 & 0.3441 & 0.3270 & 0.5245 & 0.5523 & 0.5164 & 0.5715  \\  
\hline  
\end{tabular}
\label{table:RealWorld_DS2}  
\end{table*}

Similar to the experiments on the UCI datasets, we first study the performance of the methods when the optimization is performed properly, i.e., when we pick the best results in terms of the cost function or the likelihood over $100$ different runs.
Tables \ref{table:RealWorld_DS1} and \ref{table:RealWorld_DS2}  show the performance of different clustering methods with respect to the  evaluation criteria  on DS1 and DS2. We observe that only \emph{Shifted Min Cut} yields high scores with respect to all  criteria. In most of the cases, \emph{Shifted Min Cut} results in the best scores. Otherwise, it is still competitive compared to the best choice.

Finally, we study the optimization variability, i.e., the average results and the respective standard deviations among the 100 runs. The results with respect to different evaluation criteria are shown in Tables \ref{table:RealWorld_DS1_detail} and \ref{table:RealWorld_DS2_detail} corresponding to DS1 and DS2. Similar to the experiments on the UCI datasets, we observe that the optimization variability results follow the same trend as the results in Tables  \ref{table:RealWorld_DS1} and \ref{table:RealWorld_DS2}. This indicates that the average results are consistent with the results obtained based on the best values of the cost function or the likelihood. On the other hand, \emph{Shifted Min Cut} yields the most promising results either in average or when choosing the best solutions in terms of cost/likelihood.

\begin{table*}
\begin{adjustwidth}{-0.4cm}{-0.4cm}
\caption{Average performance (and the standard deviation shown in brackets) for different methods over $100$ runs with respect to  different evaluation criteria on DS1.}
\centering  
\begin{tabular}{c | c c c c c}  
\hline\hline  
 & ShiftedMinCut & GMM & $K$-means & PIC & SP \\ [0.5ex]  
\hline  
\emph{Mutual Information} & $0.5785(0.0352)$ & $\mathbf{0.5922(0.0294)}$ & $0.5773(0.0467)$ & $0.5367(0.0386)$ & $0.3493(0.0460)$  \\  
\emph{Rand score} & $\mathbf{0.4239(0.0506)}$ & $0.2872(0.0377)$ & $0.2278(0.0491)$ & $0.2614(0.0195)$ & $0.1146(0.0207)$  \\  
\emph{V-measure} & $\mathbf{0.7914(0.0345)}$ & $0.7806(0.0272)$ & $0.7590(0.0338)$ & $0.7568(0.0248)$ & $0.5506(0.0512)$  \\ 
\hline  
\end{tabular}
\label{table:RealWorld_DS1_detail}  
\end{adjustwidth}
\end{table*}

\begin{table*}
\begin{adjustwidth}{-0.4cm}{-0.4cm}
\caption{Average performance (and the standard deviation)  for different methods over $100$ runs with respect to  different evaluation criteria on DS2. On both DS1 and DS2, \emph{Shifted Min Cut} yields more promising results in overall.}
\centering  
\begin{tabular}{c | c c c c c}  
\hline\hline  
 & ShiftedMinCut & GMM & $K$-means & PIC & SP \\ [0.5ex]  
\hline  
\emph{Mutual Information} & $\mathbf{0.5617(0.0218)}$ & $0.4427(0.0465)$ & $0.4519(0.0421)$ & $0.4185(0.0307)$ & $0.5118(0.0535)$  \\  
\emph{Rand score} & $0.5185(0.0266)$ & $0.4502(0.0379)$ & $0.4327(0.0289)$ & $0.4151(0.0342)$ & $\mathbf{0.5424}(0.0372)$  \\  
\emph{V-measure} & $\mathbf{0.5891(0.0335)}$ & $0.5081(0.0507)$ & $0.5122(0.0410)$ & $0.4870(0.0441)$ & $0.5385(0.0463)$  \\ 
\hline  
\end{tabular}
\label{table:RealWorld_DS2_detail} 
\end{adjustwidth}
\end{table*}

\section{Conclusion}
This paper investigates an alternative approach for regularizing the \emph{Min Cut} cost function in order to avoid the  appearance of singleton clusters, where the regularization term is added to the cost function, instead of dividing the \emph{Min Cut} clusters by a cluster dependent factor. We, in particular, studied the case where the regularization term leads to subtracting the pairwise similarities by the regularization factor. Then, we only need to apply the base \emph{Min Cut}, but on the (adaptively) shifted similarities instead of the original data. In the following, we developed an efficient \emph{local search} algorithm to optimize (locally) the \emph{Shifted Min Cut} cost function and studied its fast theoretical convergence rate. Thereafter, we discussed that unlike \emph{Min Cut}, many other common clustering cost functions are invariant with respect to the shift of pairwise similarities. Finally, we performed extensive experiments on several UCI and real-world datasets to demonstrate the superior performance of \emph{Shifted Min Cut} according to different evaluation criteria.

\section*{Acknowledgement}
This work was partially supported by the Wallenberg AI, Autonomous Systems and Software Program (WASP) funded by the Knut and Alice Wallenberg Foundation.
Parts of this work have been done at Xerox Research Centre Europe (Naver Labs Europe).

\bibliographystyle{plain}
\bibliography{references}

\begin{thebibliography}{10}

\bibitem{Bailey1994}
Kenneth~D. Bailey.
\newblock {\em Numerical Taxonomy and Cluster Analysis}.
\newblock SAGE Publications, 1994.

\bibitem{BansalBC04}
Nikhil Bansal, Avrim Blum, and Shuchi Chawla.
\newblock Correlation clustering.
\newblock {\em Machine Learning}, 56(1-3):89--113, 2004.

\bibitem{Buhler:2009PSpec}
Thomas B\"{u}hler and Matthias Hein.
\newblock Spectral clustering based on the graph p-laplacian.
\newblock In {\em Proceedings of the 26th Annual International Conference on
  Machine Learning}, ICML '09, pages 81--88. ACM, 2009.

\bibitem{BuloPB11}
Samuel~Rota Bul{\`{o}}, Marcello Pelillo, and Immanuel~M. Bomze.
\newblock Graph-based quadratic optimization: {A} fast evolutionary approach.
\newblock {\em Computer Vision and Image Understanding}, 115(7):984--995, 2011.

\bibitem{Cattell1943}
Raymond~B. Cattell.
\newblock {The description of personality: basic traits resolved into
  clusters}.
\newblock {\em The Journal of Abnormal and Social Psychology}, 38(4):476--506,
  1943.

\bibitem{ChanSZ94}
Pak~K. Chan, Martine D.~F. Schlag, and Jason~Y. Zien.
\newblock Spectral k-way ratio-cut partitioning and clustering.
\newblock {\em IEEE Trans. on CAD of Integrated Circuits and Systems},
  13(9):1088--1096, 1994.

\bibitem{mlChehreghani16}
Morteza~Haghir Chehreghani.
\newblock Adaptive trajectory analysis of replicator dynamics for data
  clustering.
\newblock {\em Machine Learning}, 104(2-3):271--289, 2016.

\bibitem{ChehreghaniIJCNN2021}
Morteza~Haghir Chehreghani.
\newblock Reliable agglomerative clustering.
\newblock In {\em International Joint Conference on Neural Networks (IJCNN)}.
  IEEE, 2021.

\bibitem{ChehreghaniAC08}
Morteza~Haghir Chehreghani, Hassan Abolhassani, and Mostafa~Haghir Chehreghani.
\newblock Improving density-based methods for hierarchical clustering of web
  pages.
\newblock {\em Data Knowl. Eng.}, 67(1):30--50, 2008.

\bibitem{ChehreghaniBB12}
Morteza~Haghir Chehreghani, Alberto~Giovanni Busetto, and Joachim~M. Buhmann.
\newblock Information theoretic model validation for spectral clustering.
\newblock In {\em Proceedings of the Fifteenth International Conference on
  Artificial Intelligence and Statistics, {AISTATS}}, volume~22, pages
  495--503, 2012.

\bibitem{ChenZJ05}
Yixin Chen, Ya~Zhang, and Xiang Ji.
\newblock Size regularized cut for data clustering.
\newblock In {\em Advances in Neural Information Processing Systems 18 (NIPS)},
  pages 211--218, 2005.

\bibitem{DemaineEFI06}
Erik~D. Demaine, Dotan Emanuel, Amos Fiat, and Nicole Immorlica.
\newblock Correlation clustering in general weighted graphs.
\newblock {\em Theor. Comput. Sci.}, 361(2-3):172--187, 2006.

\bibitem{abs-2007-14524}
Andreas Demetriou, Henrik Alfsv{\aa}g, Sadegh Rahrovani, and Morteza~Haghir
  Chehreghani.
\newblock A deep learning framework for generation and analysis of driving
  scenario trajectories.
\newblock {\em CoRR}, abs/2007.14524, 2020.

\bibitem{Dhillon:2004}
Inderjit~S. Dhillon, Yuqiang Guan, and Brian Kulis.
\newblock Kernel k-means: Spectral clustering and normalized cuts.
\newblock In {\em Proceedings of the Tenth ACM SIGKDD International Conference
  on Knowledge Discovery and Data Mining}, KDD '04, pages 551--556. ACM, 2004.

\bibitem{Dhillon:EtAl:05}
Inderjit~S. Dhillon, Yuqiang Guan, and Brian Kulis.
\newblock A unified view of kernel k-means, spectral clustering and graph cuts.
\newblock Technical Report TR-04-25, 2005.

\bibitem{DING202028}
Hu~Ding.
\newblock Faster balanced clusterings in high dimension.
\newblock {\em Theoretical Computer Science}, 842:28--40, 2020.

\bibitem{DBSCAN}
Martin Ester, Hans-Peter Kriegel, J\"{o}rg Sander, and Xiaowei Xu.
\newblock A density-based algorithm for discovering clusters in large spatial
  databases with noise.
\newblock In {\em Proceedings of the Second International Conference on
  Knowledge Discovery and Data Mining (KDD)}, page 226–231, 1996.

\bibitem{FrankCB11}
Mario Frank, Morteza~Haghir Chehreghani, and Joachim~M. Buhmann.
\newblock The minimum transfer cost principle for model-order selection.
\newblock In {\em European Conference on Machine Learning and Knowledge
  Discovery in Databases (ECML-PKDD)}, Lecture Notes in Computer Science, pages
  423--438, 2011.

\bibitem{goldschmidt94}
Olivier Goldschmidt and Dorit~S. Hochbaum.
\newblock A polynomial algorithm for the k-cut problem for fixed k.
\newblock {\em Mathematics of Operations Research}, 19(1):24--37, 1994.

\bibitem{phdthesis_Morteza}
Morteza Haghir~Chehreghani.
\newblock {\em Information-theoretic validation of clustering algorithms}.
\newblock PhD thesis, ETH Zurich, 2013.

\bibitem{ChehreghaniICDM17}
Morteza Haghir~Chehreghani.
\newblock Clustering by shift.
\newblock In {\em 2017 {IEEE} International Conference on Data Mining, {ICDM}},
  pages 793--798, 2017.

\bibitem{8490741}
Junwei Han, Hanyang Liu, and Feiping Nie.
\newblock A local and global discriminative framework and optimization for
  balanced clustering.
\newblock {\em IEEE Transactions on Neural Networks and Learning Systems},
  30(10):3059--3071, 2019.

\bibitem{HofmannB97}
Thomas Hofmann and Joachim~M. Buhmann.
\newblock Pairwise data clustering by deterministic annealing.
\newblock {\em {IEEE} Trans. Pattern Anal. Mach. Intell.}, 19(1):1--14, 1997.

\bibitem{hubert1985comparing}
L.~Hubert and P.~Arabie.
\newblock {Comparing partitions}.
\newblock {\em Journal of classification}, 2(1):193--218, 1985.

\bibitem{Karger:1996}
David~R. Karger and Clifford Stein.
\newblock A new approach to the minimum cut problem.
\newblock {\em J. ACM}, 43(4):601--640, 1996.

\bibitem{lance67hierarchical}
G.~N. Lance and W.~T. Williams.
\newblock A general theory of classificatory sorting strategies.
\newblock {\em The Computer Journal}, 9(4):373--380, 1967.

\bibitem{MinCutLeighton}
Tom Leighton and Satish Rao.
\newblock Multicommodity max-flow min-cut theorems and their use in designing
  approximation algorithms.
\newblock {\em J. ACM}, 46(6):787–832, 1999.

\bibitem{Lichman:2013}
M.~Lichman.
\newblock {UCI} machine learning repository, 2013.

\bibitem{LinC10PIC}
Frank Lin and William~W. Cohen.
\newblock Power iteration clustering.
\newblock In {\em Proceedings of the 27th International Conference on Machine
  Learning (ICML-10)}, pages 655--662, 2010.

\bibitem{ijcai2019-414}
Weibo Lin, Zhu He, and Mingyu Xiao.
\newblock Balanced clustering: A uniform model and fast algorithm.
\newblock In {\em Proceedings of the Twenty-Eighth International Joint
  Conference on Artificial Intelligence (IJCAI)}, pages 2987--2993.
  International Joint Conferences on Artificial Intelligence Organization,
  2019.

\bibitem{LiuLY13}
Hairong Liu, Longin~Jan Latecki, and Shuicheng Yan.
\newblock Fast detection of dense subgraphs with iterative shrinking and
  expansion.
\newblock {\em {IEEE} Trans. Pattern Anal. Mach. Intell.}, 35(9):2131--2142,
  2013.

\bibitem{LiuHNL17}
Hanyang Liu, Junwei Han, Feiping Nie, and Xuelong Li.
\newblock Balanced clustering with least square regression.
\newblock In {\em Proceedings of the Thirty-First {AAAI} Conference on
  Artificial Intelligence}, pages 2231--2237. {AAAI} Press, 2017.

\bibitem{8621917}
Hongfu Liu, Ziming Huang, Qi~Chen, Mingqin Li, Yun Fu, and Lintao Zhang.
\newblock Fast clustering with flexible balance constraints.
\newblock In {\em IEEE International Conference on Big Data (Big Data)}, pages
  743--750, 2018.

\bibitem{Luxburg:2007}
Ulrike Luxburg.
\newblock A tutorial on spectral clustering.
\newblock {\em Statistics and Computing}, 17(4):395--416, 2007.

\bibitem{Macqueen67somemethods}
J.~Macqueen.
\newblock Some methods for classification and analysis of multivariate
  observations.
\newblock In {\em 5-th Berkeley Symposium on Mathematical Statistics and
  Probability}, pages 281--297, 1967.

\bibitem{balancedKmean2014}
Mikko~I. Malinen and Pasi Fr{\"a}nti.
\newblock Balanced k-means for clustering.
\newblock In Pasi Fr{\"a}nti, Gavin Brown, Marco Loog, Francisco Escolano, and
  Marcello Pelillo, editors, {\em Structural, Syntactic, and Statistical
  Pattern Recognition}, pages 32--41, Berlin, Heidelberg, 2014.

\bibitem{manning2008introduction}
Christopher~D. Manning, Prabhakar Raghavan, and Hinrich Schütze.
\newblock {\em Introduction to Information Retrieval}.
\newblock Cambridge University Press, 2008.

\bibitem{Ng01onspectral}
Andrew~Y. Ng, Michael~I. Jordan, and Yair Weiss.
\newblock On spectral clustering: Analysis and an algorithm.
\newblock In {\em Advances in Neural Information Processing Systems 14}, pages
  849--856. MIT Press, 2001.

\bibitem{NgMA12}
Bernard Ng, Martin~J. McKeown, and Rafeef Abugharbieh.
\newblock Group replicator dynamics: A novel group-wise evolutionary approach
  for sparse brain network detection.
\newblock {\em IEEE Trans. Med. Imaging}, 31(3):576--585, 2012.

\bibitem{PavanPHierarchy03}
Massimiliano Pavan and Marcello Pelillo.
\newblock Dominant sets and hierarchical clustering.
\newblock In {\em 9th {IEEE} International Conference on Computer Vision
  {(ICCV})}, pages 362--369, 2003.

\bibitem{Pavan:2007}
Massimiliano Pavan and Marcello Pelillo.
\newblock Dominant sets and pairwise clustering.
\newblock {\em IEEE Trans. Pattern Anal. Mach. Intell.}, 29(1):167--172, 2007.

\bibitem{Reddi16}
Sashank~J. Reddi, Suvrit Sra, Barnab{\'{a}}s P{\'{o}}czos, and Alexander~J.
  Smola.
\newblock Stochastic frank-wolfe methods for nonconvex optimization.
\newblock In {\em 54th Annual Allerton Conference on Communication, Control,
  and Computing, Allerton 2016, Monticello, IL, USA, September 27-30, 2016},
  pages 1244--1251, 2016.

\bibitem{RosenbergH07}
Andrew Rosenberg and Julia Hirschberg.
\newblock V-measure: A conditional entropy-based external cluster evaluation
  measure.
\newblock In {\em EMNLP-CoNLL}, pages 410--420. ACL, 2007.

\bibitem{roth03PAMI}
Volker Roth, Julian Laub, Motoaki Kawanabe, and Joachim~M. Buhmann.
\newblock Optimal cluster preserving embedding of nonmetric proximity data.
\newblock {\em {IEEE} Trans. Pattern Anal. Mach. Intell.}, 25(12):1540--1551,
  2003.

\bibitem{Scholkopf:1998}
Bernhard Sch\"{o}lkopf, Alexander Smola, and Klaus-Robert M\"{u}ller.
\newblock Nonlinear component analysis as a kernel eigenvalue problem.
\newblock {\em Neural Comput.}, 10(5):1299--1319, 1998.

\bibitem{SchusterRD83}
Peter Schuster and Karl Sigmund.
\newblock Replicator dynamics.
\newblock {\em J. Theor. Biol.}, 100:533--538, 1983.

\bibitem{JS:JM:PAMI:2000}
Jianbo Shi and J.~Malik.
\newblock Normalized cuts and image segmentation.
\newblock {\em IEEE Trans. on Pattern Analysis and Machine Intelligence},
  22(8):888--905, 2000.

\bibitem{sneath1957dn09j}
Peter Henry~Andrews Sneath.
\newblock The application of computers to taxonomy.
\newblock {\em Journal of General Microbiology}, 17:201--226, 1957.

\bibitem{sokal58}
R.~R. Sokal and C.~D. Michener.
\newblock A statistical method for evaluating systematic relationships.
\newblock {\em University of Kansas Science Bulletin}, 38:1409--1438, 1958.

\bibitem{SoundararajanS01}
Padmanabhan Soundararajan and Sudeep Sarkar.
\newblock Investigation of measures for grouping by graph partitioning.
\newblock In {\em Proc. Conf. Computer Vision and Pattern Recognition (CVPR)},
  pages 239--246, 2001.

\bibitem{ThielCD19}
Erik Thiel, Morteza~Haghir Chehreghani, and Devdatt~P. Dubhashi.
\newblock A non-convex optimization approach to correlation clustering.
\newblock In {\em The Thirty-Third {AAAI} Conference on Artificial
  Intelligence, {AAAI}}, pages 5159--5166, 2019.

\bibitem{tryon1939cluster}
R.C. Tryon.
\newblock {\em Cluster Analysis: Correlation Profile and Orthometric (factor)
  Analysis for the Isolation of Unities in Mind and Personality}.
\newblock Edwards brother, Incorporated, lithoprinters and publishers, 1939.

\bibitem{Vinh:2010}
Nguyen~Xuan Vinh, Julien Epps, and James Bailey.
\newblock Information theoretic measures for clusterings comparison: Variants,
  properties, normalization and correction for chance.
\newblock {\em J. Mach. Learn. Res.}, 11:2837--2854, 2010.

\bibitem{Inchoate:Ward63}
Joe~H. Ward.
\newblock Hierarchical grouping to optimize an objective function.
\newblock {\em Journal of the American Statistical Association},
  58(301):236--244, 1963.

\bibitem{GVK97}
{Jorgen W.} Weibull.
\newblock {\em Evolutionary game theory}.
\newblock MIT Press, Cambridge, Mass. [u.a.], 1997.

\bibitem{MinCut-Wu}
Z.~Wu and R.~Leahy.
\newblock An optimal graph theoretic approach to data clustering: theory and
  its application to image segmentation.
\newblock {\em IEEE Transactions on Pattern Analysis and Machine Intelligence},
  15(11):1101--1113, 1993.

\bibitem{YangCL019}
Linxiao Yang, Ngai{-}Man Cheung, Jiaying Li, and Jun Fang.
\newblock Deep clustering by gaussian mixture variational autoencoders with
  graph embedding.
\newblock In {\em International Conference on Computer Vision (ICCV)}, pages
  6439--6448, 2019.

\end{thebibliography}

\end{document}